\documentclass{article}

\usepackage{microtype}
\usepackage{graphicx}
\usepackage{subcaption}
\usepackage{booktabs} 
\usepackage{multirow}

\usepackage{hyperref}

\usepackage{icml2026}
\makeatletter
\renewcommand{\Notice@String}{}
\icmlshowauthorstrue
\makeatother

\usepackage{amsmath}
\usepackage{amssymb}
\usepackage{mathtools}
\usepackage{amsthm}

\usepackage[capitalize,noabbrev]{cleveref}

\theoremstyle{plain}
\newtheorem{theorem}{Theorem}[section]

\newtheorem{lemma}[theorem]{Lemma}

\theoremstyle{definition}
\newtheorem{definition}[theorem]{Definition}
\newtheorem{assumption}[theorem]{Assumption}
\theoremstyle{remark}

\usepackage[textsize=tiny]{todonotes}

\icmltitlerunning{ASPECT:Analogical Semantic Policy Execution via Language Conditioned Transfer}

\begin{document}

\twocolumn[
  \icmltitle{ASPECT:Analogical Semantic Policy Execution via Language Conditioned Transfer}



  \icmlsetsymbol{equal}{*}

  \begin{icmlauthorlist}
    \icmlauthor{Ajsal Shereef Palattuparambil}{a3i}
    \icmlauthor{Thommen George Karimpanal}{sit}
    \icmlauthor{Santu Rana}{a3i}
  \end{icmlauthorlist}

  \icmlaffiliation{a3i}{Applied Artificial Intelligence Initiative, Deakin University, Waurn Ponds, Geelong}
  \icmlaffiliation{sit}{School of IT, Deakin University, Waurn Ponds, Geelong}

  \icmlcorrespondingauthor{Ajsal Shereef Palattuparambil}{a.palattuparambil@deakin.edu.au}

  \icmlkeywords{Reinforcement Learning, Zero-shot Transfer, Analogical Reasoning, Language-Conditioned Policies}

  \vskip 0.3in
]



\printAffiliationsAndNotice{}  

\begin{abstract}
Reinforcement Learning (RL) agents often struggle to generalize knowledge to new tasks, even those structurally similar to ones they have mastered. Although recent approaches have attempted to mitigate this issue via zero-shot transfer, they are often constrained by predefined, discrete class systems, limiting their adaptability to novel or compositional task variations. We propose a significantly more generalized approach, replacing discrete latent variables with natural language conditioning via a text-conditioned Variational Autoencoder (VAE). Our core innovation utilizes a Large Language Model (LLM) as a dynamic \textit{semantic operator} at test time. Rather than relying on rigid rules, our agent queries the LLM to semantically remap the description of the current observation to align with the source task. This source-aligned caption conditions the VAE to generate an imagined state compatible with the agent's original training, enabling direct policy reuse. By harnessing the flexible reasoning capabilities of LLMs, our approach achieves zero-shot transfer across a broad spectrum of complex and truly novel analogous tasks, moving beyond the limitations of fixed category mappings. Code and videos are available \href{https://anonymous.4open.science/r/ASPECT-85C3/}{here}.
\end{abstract}

\section{Introduction}

Humans possess a remarkable capacity for analogical reasoning, enabling us to adapt existing knowledge to new, structurally similar tasks with minimal/zero relearning. If we learn to pick an apple, we can intuitively apply that same skill to pick an orange without extensive trial and error. This fluid transfer of knowledge remains a fundamental challenge for Deep Reinforcement Learning agents, which typically require costly retraining even when faced with minor changes to a task's goals or semantics \cite{zhang2018dissection}.

A promising direction for closing this gap is to equip agents with a mechanism for ``imagination'' \cite{hafner2019dream, nair2018visual}. Recent work MAGIK \cite{palattuparambil2025magik} demonstrated a novel framework for zero-shot knowledge transfer by training an agent to imagine analogous goals. It learns to disentangle an observation's task-agnostic features (e.g., an object's position) from its task-specific features (e.g., its identity as an ``apple'') into a discrete class latent $c$. By programmatically swapping the latent class of a target object (e.g., ``orange'') with that of a source object (``apple''), the authors demonstrated how an agent can generate a source-aligned observation and reuse its original policy.

However, relying on a predefined, discrete set of classes fundamentally constrains the agent's analogical capabilities. A critical limitation of this framework is its inability to extrapolate to unseen tasks that share the same action affordance as the source. For instance, while prior work can map a known ``orange'' to an ``apple,'' it fails to transfer the policy for picking an ``apple'' to a previously \textit{unseen} ``banana.'' This rigidity severely limits deployment in real-world environments where novel objects are ubiquitous. Such scenarios demand a semantic understanding that extends beyond fixed classification systems. To overcome this bottleneck, we introduce \textbf{A}nalogical \textbf{S}emantic \textbf{P}olicy \textbf{E}xecution via Language \textbf{C}onditioned \textbf{T}ransfer (ASPECT), a method grounded in natural language. By leveraging the rich semantic context provided by language, our framework can map a known source policy to \textit{any} target task, including those involving unseen objects, provided they share the same core action affordance. This significantly improves the generalisability of RL policies, allowing them to adapt to open-world settings without costly retraining.

Specifically, we generalize the imagination-based transfer framework by replacing the discrete class latent $c$ with a continuous latent space conditioned on natural language. Rather than training a VAE on predefined classes, we utilize a text-conditioned VAE trained on image-text pairs. This approach enables the model to disentangle structural features $z$ from high-dimensional semantic content derived from textual descriptions.

Our approach uses a Large Language Model (LLM) as a dynamic semantic operator at test time. In contrast to the original MAGIK framework, which relied on researcher-defined, fixed transformations (e.g., ``map red ball to green ball''), our method leverages the LLM's reasoning capabilities to autonomously determine semantic remappings. By prompting the LLM with a description of the current observation, the agent identifies how to align the novel target task with its existing knowledge base. For instance, the LLM might deduce that, within the context of the agent's goal, a target ``orange'' is semantically analogous to a source ``apple.''

This source-aligned description conditions the generative model, allowing the agent to reconstruct the current state by combining the latent feature $z$ of the current state with the LLM-provided remapped captions. Effectively, the agent ``reimagines'' the novel target scene in the familiar terms of its source task, enabling direct, zero-shot application of the original policy. Integrating generative imagination with the flexible reasoning of LLMs enables our approach to transcend simple one-to-one mappings, facilitating complex, compositional, and truly novel analogical knowledge transfer.

Our contributions are as follows:
\begin{itemize}
    \item We propose a novel framework that leverages a text-conditioned VAE and an LLM to achieve flexible, zero-shot policy transfer.
    \item We introduce the concept of an LLM as a "semantic operator" to dynamically map target observations to source-task analogues, replacing the rigid, discrete class system of prior work.
    \item We demonstrate that our approach can generalize to a significantly wider and more complex range of semantic tasks, including those involving unseen objects.
\end{itemize}

\section{Related Work}

Our work, ASPECT, is positioned at the intersection of several key research areas in reinforcement learning, including transfer learning, zero-shot generalization, relational reasoning, and imagination-based policy execution.

\subsection{Transfer Learning and Domain Adaptation in RL}
The challenge of transferring knowledge across tasks is a long-standing problem in reinforcement learning. Prominent approaches include Successor Features (SFs) \cite{barreto2017successor,chua2024learning}, which learn representations that decouple environment dynamics from reward functions to accelerate transfer between tasks with different goals. However, these and other traditional transfer learning methods still require a phase of online interaction and fine-tuning, especially when task structures or dynamics change. 

A related field, domain adaptation, attempts to learn policies that are robust to shifts in observations, often by learning domain-invariant features \cite{gamrian2019transfer}. While effective, these methods typically require access to data from the target domain, limiting transferability, and do not explicitly model the semantic analogies between task components.





\subsection{Imagination, Planning, and World Models}
The concept of ``imagination'' is most prominently featured in model-based RL. Works like Dreamer \cite{hafner2019dream} learn a latent-space world model and then train a policy by imagining future trajectories entirely within this learned model, leading to high sample efficiency. Other approaches have used imagination to generate and select goals \cite{nair2018visual}. 

The imagination scheme used by \cite{palattuparambil2025magik}, and adopted by our approach, is distinct. It does not imagine future states based on environment dynamics; instead, it ``imagines'' a translation, an analogical mapping between the components of two different environments. Our method generalizes this by conditioning the translation on natural language, rather than the confined, predefined discrete latents of the original work.

\subsection{Relational and Analogical Reasoning}
The method is fundamentally inspired by human analogical reasoning. This has been most closely studied in Relational Reinforcement Learning (RRL) \cite{zambaldi2018relational}, which aims to learn policies that operate over objects and their relations, rather than raw pixel data. This relational structure is a powerful prior for generalization. Our work differs by focusing on the transfer problem: rather than learning a single, general relational policy from scratch, we assume a high-performing policy exists in a source domain and focus on translating a new task into the source policy's language.

\section{Preliminaries and Background}

This section provides an overview of the core concepts that form the foundation of our work: Reinforcement Learning (RL), the training of a source policy, Variational Autoencoders (VAEs), and the principles of language-conditioned generative models.

\subsection{Reinforcement Learning (RL)}
We formulate our problem within the Reinforcement Learning (RL) framework. An RL environment is typically modeled as a Markov Decision Process (MDP), defined by the tuple $(\mathcal{S}, \mathcal{A}, \mathcal{T}, \mathcal{R}, \gamma)$. Here, $\mathcal{S}$ is the space of all possible states, $\mathcal{A}$ is the set of actions, $\mathcal{T}(s'|s, a)$ is the state transition function defining the probability of transitioning to state $s'$ from state $s$ after taking action $a$, $\mathcal{R}(s, a)$ is the reward function, and $\gamma \in [0, 1]$ is the discount factor.

The goal of an RL agent is to learn a policy, $\pi(a|s)$, which is a mapping from states to a distribution over actions. The policy is optimized to maximize the expected cumulative discounted return $G_t = \mathbb{E}[\sum_{k=0}^{\infty} \gamma^k R_{t+k+1}]$, which represents the total accumulated reward over time.

\subsection{Source Policy Learning}
In our framework, the agent's source policy $\pi_{source}$ is trained in the source environment to solve a specific task, learning to maximize the expected cumulative reward $G_t$. The core of our methodology is agnostic to the specific algorithm used to train $\pi_{source}$. Any standard deep RL algorithm can be employed.

To demonstrate this flexibility, our experiments utilize a variety of prominent algorithms, including on-policy methods like Proximal Policy Optimization (PPO) \cite{schulman2017proximal} and off-policy methods like Deep Q-Networks (DQN) \cite{mnih2015human} (for discrete actions) and Soft Actor-Critic (SAC) \cite{haarnoja2018soft} (for continuous actions).

\subsection{Language-Conditioned Generative Models}
Generative models, such as Variational Autoencoders (VAEs) \cite{kingma2013auto}, learn a compressed, probabilistic latent representation that captures the underlying structure of data. While traditional VAEs generate data from this latent space alone, conditional variants allow for control over the generative process through auxiliary information.

In the context of language conditioning, modern approaches leverage rich embeddings from pre-trained language models (like CLIP \cite{radford2021learning}) to condition generation on free-form text descriptions. This text conditioning is typically integrated into the model architecture using mechanisms like FiLM (Feature-wise Linear Modulation) \cite{perez2018film} or Cross-Attention \cite{vaswani2017attention}, enabling fine-grained semantic control over the synthesized output. Our work builds upon these foundations to enable language-guided imagination.

\subsection{LLMs as Reasoning Engines}
LLMs \cite{brown2020language} are Transformer-based models trained on web-scale text corpora. While LLMs excel at text generation, their most significant capability for our work is their emergent capacity for in-context learning and zero-shot reasoning. Given a prompt containing a few examples or a structured set of instructions (a ``context''), LLMs can perform novel tasks without any gradient updates. This allows them to function as flexible ``semantic operators,'' capable of translating concepts, performing analogical reasoning, and rephrasing information from one domain to another based on the provided context. We leverage this capability to build our mapping function $M_{LLM}$ (More details in Section \ref{sec:llm_operator}).

\section{Methodology}
\label{sec:methodology}
Our objective is to enable an RL agent, pre-trained on a source task, to perform zero-shot transfer to novel, analogous target tasks (formally defined in Definition \ref{def:analogous_tasks}) using an imagination-based mechanism guided by natural language. Building on prior work \cite{palattuparambil2025magik}, which used a discrete class system, we introduce a more flexible framework centered around a text-conditioned VAE and an LLM for semantic reasoning.

Let $\pi_{source}(a|s)$ denote the policy learned by an RL agent (e.g., using SAC \cite{haarnoja2018soft}) for a source task, operating on states $s$ from the source environment $S^S$. Our goal is to derive an effective policy $\pi_{target}(a|s_t)$ for a target task in environment $S^T$, where $s_t \in S^T$, without any direct interaction or training in the target environment. We achieve this by learning a sophisticated imagination function $\Psi_{LLM}(s_t, c_{t \to s})$ that maps a target state $s_t$ and the LLM-generated source-aligned caption to an imagined source-aligned state $s_{imagined}$. The target policy is then defined as $\pi_{target}(a|s_t) = \pi_{source}(a | \Psi_{LLM}(s_t, c_{t \to s}))$. See Section \ref{sec:llm_operator} for the definition of $c_{t \to s}$. In this context, the term ``zero-shot'' refers to the fact that our method is applied directly to the target task without any re-interaction with the environment and without requiring any additional data collection or training in the target setting.

The core components of our methodology are: (1) a text-conditioned VAE trained to reconstruct states based on semantic descriptions, and (2) an LLM acting as a semantic operator to translate target task descriptions into source task analogues at test time.

\subsection{Text-Conditioned VAE for Imagination}

We employ a VAE architecture modified to condition the generative process on natural language descriptions. This VAE is trained offline on a dataset $\mathcal{D} = \{(x_i, c_{s,i})\}_{i=1}^N$ consisting of observations $x_i$ (e.g., image frames) paired with corresponding source captions $c_{s,i}$ describing the scene content. These pairs are collected during the agent's interaction with the source environment while learning $\pi_{source}$. Note that $c_{s,i}$ describes the scene itself and does not inherently contain information about the agent's task.

\subsubsection{Architecture and Latent Space}
The VAE consists of an encoder network $q_{\phi}(z|x)$ and a decoder network $p_{\theta}(x|z, c_s)$. The encoder maps an input observation $x$ to a distribution over a continuous latent variable $z \in \mathbb{R}^k$. The decoder reconstructs the observation $\hat{x}$ by sampling from $z$ and using the provided source caption $c_s$ as guidance for the scene's semantic content.

Specifically, the source caption $c_s$ is first embedded into a continuous vector $e_{c_s} = E_{text}(c_s)$ using a pre-trained text embedding model, such as the text encoder from CLIP (e.g., `clip-vit-base-patch32'). This text embedding $e_{c_s}$ then conditions the decoding process $p_{\theta}(x|z, e_{c_s})$. We implement this conditioning using a combination of Feature-wise Linear Modulation (FiLM) and Cross-Attention layers within the decoder architecture. FiLM layers modulate the intermediate feature maps of the decoder based on the latent code $z$, while Cross-Attention layers allow the decoder features to attend to the text embedding $e_{c_s}$, integrating the semantic guidance provided by the caption into the reconstruction process. This ensures that the generated image $\hat{x}$ reflects both the structural information from $z$ and the semantic content specified by $c_s$.

\subsubsection{Training Objective and Disentanglement}
The VAE is trained to maximize a modified Evidence Lower Bound (ELBO) objective, incorporating both pixel-wise Mean Squared Error and perceptual loss for high-fidelity reconstruction, regularized by a $\beta$-weighted KL-divergence term \cite{higgins2017beta}. Crucially, to ensure effective zero-shot transfer, we strictly disentangle the spatial latent $z$ from the semantic text embedding $e_{c_s}$. We employ an adversarial training scheme with a Gradient Reversal Layer (GRL) \cite{ganin2016domain}, where a discriminator $D_{\psi}$ learns to predict $e_{c_s}$ from $z$ via a contrastive InfoNCE loss. The encoder is simultaneously updated to maximize this loss, guarding against semantic leakage and ensuring $z$ captures only structural features orthogonal to the text. The complete training objective is detailed in Appendix \ref{app:vae_details}, and we discuss the disentanglement in Appendix \ref{app:disentanglement}.

\begin{figure}[hbt]
    \centering
    \includegraphics[width=0.8\linewidth]{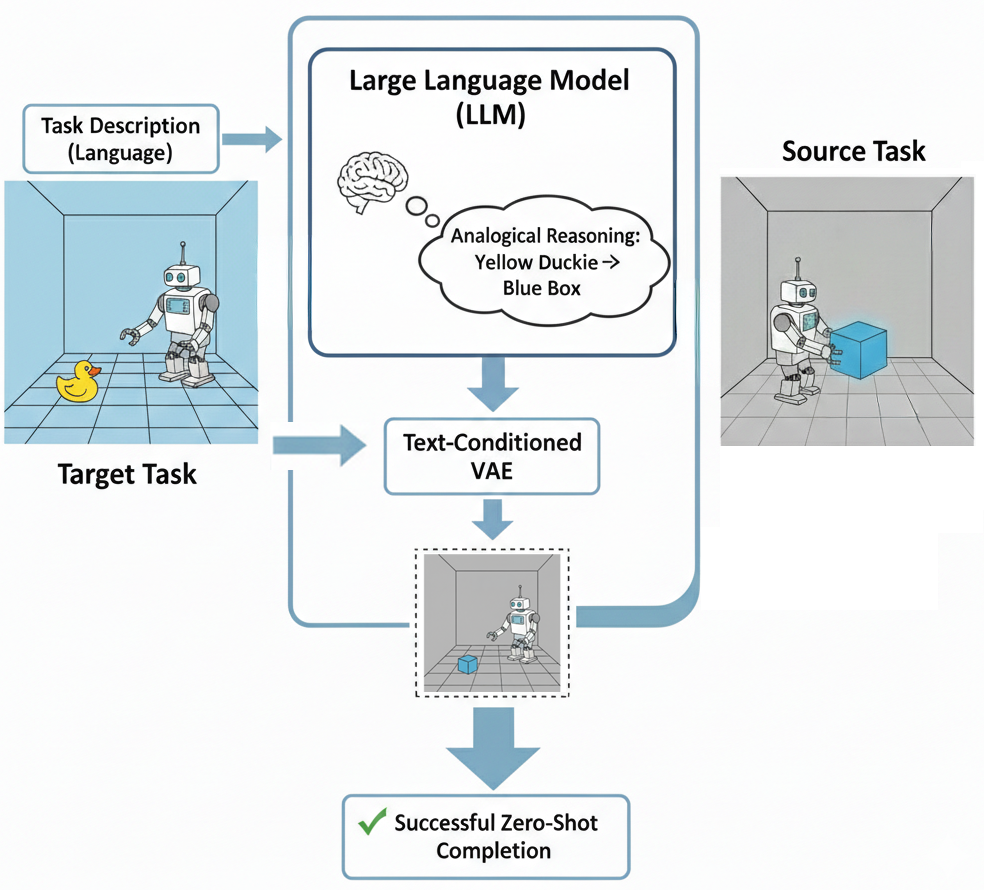}
    \caption{The overall idea of ASPECT. The agent uses an LLM to semantically remap the target observation (e.g., target task: ``pick yellow duckie from blue room'') to a source-aligned description (e.g., source task: ``pick blue box from grey room''), enabling the direct application of the pre-trained source policy to the target task.}
    \label{fig:idea}
\end{figure}

\subsection{LLM as a Semantic Operator for Analogical Mapping}
\label{sec:llm_operator}
The key component enabling flexible, zero-shot transfer is the use of an LLM as a semantic operator, $M_{LLM}$. At test time, when the agent encounters a state $s_t$ in the target task, we first obtain a caption $c_t$ of this state. This caption can be generated by a Vision-Language Model (VLM), provided externally, or follow a template based on object detectors.

The function of the LLM is to translate this target caption $c_t$ into a source-aligned caption $c_{t \to s}$. The LLM manipulates the caption $c_t$ based on the source and target task goals, effectively finding a semantic analogy that aligns the current target situation with the agent's prior experience from the source task. This process can be formulated as \(c_{t \to s} = M_{LLM}(c_t, \mathcal{C})\),
where $\mathcal{C}$ represents the context provided to the LLM used to retrieve $c_{t \to s}$ from $c_t$. This context comprises a description of the environment, the target task (e.g., ``Pick yellow duckie''), the source task the agent was trained on (e.g., ``Pick blue box''), and the caption $c_t$ of the current observation (e.g., ``A yellow duckie is on the table'').
Along with this context, a directive query is provided (see Appendix \ref{app:llm_details}). The LLM then outputs the manipulated, source-aligned caption $c_{t \to s}$ (e.g., ``A blue box is on the table'').

This querying leverages the LLM's ability to perform analogical reasoning on-the-fly. Significantly, this approach allows the agent to potentially handle target tasks involving objects or concepts (like the ``yellow duckie'') that were entirely absent during the VAE training, relying on the LLM's general semantic understanding to bridge the gap. The complete zero-shot transfer procedure is summarized in Algorithm \ref{alg:llm_magik} in Appendix \ref{app:algorithm}. The overall idea of our method is illustrated in (Figure \ref{fig:idea}).

\section{Theoretical Analysis of Analogical Transfer}

In this section, we formally characterize the analogical transfer mechanism of ASPECT.
We define analogous tasks using MDP homomorphisms, prove the existence of an ideal
source-aligned state, and derive a performance degradation bound under approximate
semantic mapping.

\subsection{Preliminaries}

Let the source and target tasks be modeled as Markov Decision Processes (MDPs):

\[
\mathcal{M}_S = (S_S, A, T_S, R_S, \gamma)
\]
\[
\mathcal{M}_T = (S_T, A, T_T, R_T, \gamma)
\]

where:

\begin{itemize}
    \item $S_S$ and $S_T$ are the source and target state spaces,
    \item $A$ is the shared action space,
    \item $T_S(s' \mid s,a)$ and $T_T(s' \mid s,a)$ are transition kernels,
    \item $R_S(s,a)$ and $R_T(s,a)$ are reward functions,
    \item $\gamma \in (0,1)$ is the discount factor.
\end{itemize}

\paragraph{State Factorization.}

We assume that states admit a structural–semantic decomposition:

\[
s = (u,v)
\]

where:

\begin{itemize}
    \item $u \in \mathcal{U}$ denotes structural features
          (e.g., spatial configuration),
    \item $v \in \mathcal{V}$ denotes semantic identity
          (e.g., object class or affordance role).
\end{itemize}

Thus:

\[
S_S \subseteq \mathcal{U} \times \mathcal{V}_S,
\quad
S_T \subseteq \mathcal{U} \times \mathcal{V}_T.
\]

\subsection{Analogous Tasks via Affordance-Preserving Mapping}

\begin{definition}[Affordance-Preserving Mapping]
A mapping
\[
\Phi : \mathcal{V}_T \rightarrow \mathcal{V}_S
\]
is called affordance-preserving if for all
$u,u' \in \mathcal{U}$,
$v,v' \in \mathcal{V}_T$,
and $a \in A$:

\[
R_T((u,v),a) = R_S((u,\Phi(v)),a),
\]

\[
T_T((u',v') \mid (u,v),a)
=
T_S((u',\Phi(v')) \mid (u,\Phi(v)),a).
\]

\end{definition}

\begin{definition}[Analogous Tasks]
\label{def:analogous_tasks}
The target MDP $\mathcal{M}_T$ is analogous to
the source MDP $\mathcal{M}_S$ if there exists
an affordance-preserving mapping $\Phi$.
\end{definition}

\subsection{MDP Homomorphism}

Define the state mapping

\[
h : S_T \rightarrow S_S,
\quad
h(u,v) = (u, \Phi(v)).
\]

\begin{definition}[Exact Homomorphism]
The mapping $h$ defines an exact MDP homomorphism
if for all $s \in S_T$, $a \in A$:

\[
R_T(s,a) = R_S(h(s),a),
\]

\[
T_{T}(s^{\prime}|s,a)=T_{S}(h(s^{\prime})|h(s),a).
\]
\end{definition}

Under the structural–semantic factorization above,
this reduces to direct substitution via $\Phi$.

\subsection{Existence of Ideal Source-Aligned State}

\begin{lemma}[Existence of Ideal Aligned State]
If $\mathcal{M}_T$ is analogous to $\mathcal{M}_S$,
then for every $s_T \in S_T$,
there exists a corresponding
\[
s_{\text{ideal}} = h(s_T) \in S_S.
\]
\end{lemma}

\begin{proof}
Let $s_T = (u,v)$.
By definition of $h$:

\[
h(s_T) = (u,\Phi(v)).
\]

Since $\Phi(v) \in \mathcal{V}_S$,
we have $h(s_T) \in S_S$.
Existence follows directly from definition of $h$.
\end{proof}

\subsection{Value Functions}

For any policy $\pi : S \to \Delta(A)$,
define its value function in MDP $\mathcal{M}$ as:

\[
V_{\mathcal{M}}^{\pi}(s)
=
\mathbb{E}
\left[
\sum_{t=0}^{\infty}
\gamma^t
R_{\mathcal{M}}(s_t,a_t)
\;\middle|\;
s_0 = s,
\pi
\right].
\]

\subsection{Value Preservation Under Exact Homomorphism}

\begin{lemma}[Value Preservation]
If $h$ defines an exact homomorphism,
then for any policy $\pi_S : S_S \to \Delta(A)$,
define the induced target policy:

\[
\pi_T(s) := \pi_S(h(s)).
\]

Then for all $s_T \in S_T$:

\[
V_{\mathcal{M}_T}^{\pi_T}(s_T)
=
V_{\mathcal{M}_S}^{\pi_S}(h(s_T)).
\]
\end{lemma}

\begin{proof}
From reward preservation:

\[
R_T(s,a) = R_S(h(s),a).
\]

From transition preservation:

\[
T_T(s'|s,a)
=
T_S(h(s')|h(s),a).
\]

Thus the Bellman equations coincide:

\[
V_{\mathcal{M}_T}^{\pi_T}(s)
=
\mathbb{E}
\left[
R_T(s,a)
+
\gamma
V_{\mathcal{M}_T}^{\pi_T}(s')
\right]
\]

\[
=
\mathbb{E}
\left[
R_S(h(s),a)
+
\gamma
V_{\mathcal{M}_S}^{\pi_S}(h(s'))
\right]
=
V_{\mathcal{M}_S}^{\pi_S}(h(s)).
\]
\end{proof}

\subsection{Induced Policies and Value Functions}

Let the source and target MDPs be

\begin{align*}
\mathcal{M}_S &= (S_S, A, T_S, R_S, \gamma), \\
\mathcal{M}_T &= (S_T, A, T_T, R_T, \gamma),
\end{align*}

where $\gamma \in (0,1)$.

Let
\[
h : S_T \to S_S
\]
be an exact MDP homomorphism and
\[
\Psi : S_T \to S_S
\]
be the approximate state mapping produced by ASPECT.

\paragraph{Source Policy.}
Let
\[
\pi_S : S_S \to \Delta(A)
\]
be a stochastic policy.

\paragraph{Induced Target Policies.}
We define two target policies:

\[
(\pi_S \circ h)(a|s_T)
:=
\pi_S(a|h(s_T)),
\]

\[
\pi_{\Psi}(a|s_T)
:=
\pi_S(a|\Psi(s_T)).
\]

\paragraph{Value Function.}

For any MDP $\mathcal{M}$ and policy $\pi$:

\[
V_{\mathcal{M}}^\pi(s)
=
\mathbb{E}
\left[
\sum_{t=0}^\infty
\gamma^t R(s_t,a_t)
\;\middle|\;
s_0=s
\right].
\]

Define the Bellman operator:

\begin{align*}
(\mathcal{T}_{\mathcal{M}}^\pi V)(s)
&= \int_A \pi(a|s) \bigg[ R(s,a) \\
&\quad+ \gamma \int_S T(s'|s,a)V(s') ds' \bigg] da.
\end{align*}

\begin{lemma}[Contraction Property]
For any policy $\pi$, the Bellman operator is a $\gamma$-contraction:
\[
\|\mathcal{T}_{\mathcal{M}}^\pi V_1 - \mathcal{T}_{\mathcal{M}}^\pi V_2\|_\infty \le \gamma \|V_1 - V_2\|_\infty.
\]
Hence $\mathcal{T}_{\mathcal{M}}^\pi$ has a unique fixed point $V_{\mathcal{M}}^\pi$.
\end{lemma}

\begin{proof}
For any state $s$:
\begin{align*}
&|(\mathcal{T}_{\mathcal{M}}^\pi V_1)(s) - (\mathcal{T}_{\mathcal{M}}^\pi V_2)(s)| \\
&\quad = \bigg| \int_A \pi(a|s) \Big[ \gamma \int_S T(s'|s,a) (V_1(s') - V_2(s')) ds' \Big] da \bigg| \\
&\quad \le \gamma \int_A \pi(a|s) \int_S T(s'|s,a) \big| V_1(s') - V_2(s') \big| ds' da \\
&\quad \le \gamma \int_A \pi(a|s) \int_S T(s'|s,a) \| V_1 - V_2 \|_\infty ds' da \\
&\quad = \gamma \| V_1 - V_2 \|_\infty \int_A \pi(a|s) \Big( \int_S T(s'|s,a) ds' \Big) da.
\end{align*}

Since $T(\cdot|s,a)$ is a probability distribution over the state space, we have $\int_S T(s'|s,a) ds' = 1$. Similarly, since $\pi(\cdot|s)$ is a probability distribution over actions, we have $\int_A \pi(a|s) da = 1$. Thus:
\[
|(\mathcal{T}_{\mathcal{M}}^\pi V_1)(s) - (\mathcal{T}_{\mathcal{M}}^\pi V_2)(s)| \le \gamma \| V_1 - V_2 \|_\infty.
\]

Taking the supremum over all $s$ yields the contraction property. By the Banach fixed-point theorem, since $\gamma \in (0, 1)$, $\mathcal{T}_{\mathcal{M}}^\pi$ has a unique fixed point.
\end{proof}

\subsection{Performance Degradation Under Approximate Mapping}

Let

\[
s_{\text{ideal}} := h(s_T),
\qquad
\hat{s} := \Psi(s_T).
\]


\begin{assumption}[TV-Lipschitz Source Policy]
There exists $L_{TV} > 0$ such that for all $s_1, s_2$:
$$D_{TV}(\pi_S(\cdot|s_1), \pi_S(\cdot|s_2)) \le L_{TV} ||s_1 - s_2||$$
\end{assumption}

\begin{assumption}[Bounded Action-Value]
There exists $Q_{\max}$ such that

\[
|Q_T^{\pi_S \circ h}(s,a)| \le Q_{\max}
\quad
\forall s,a.
\]
\end{assumption}

\begin{theorem}[Performance Degradation Bound]
Suppose the source policy $\pi_S$ is $L_{TV}$-Lipschitz continuous with respect to the Total Variation distance, meaning $D_{TV}(\pi_S(\cdot|s_1), \pi_S(\cdot|s_2)) \le L_{TV} ||s_1 - s_2||$ for all $s_1, s_2 \in S_S$. 

If the approximate state mapping $\Psi$ produced by ASPECT satisfies $||\Psi(s) - h(s)|| \le \epsilon$ for any target states $s \in S_T$, then the maximum value function degradation across all states (measured in the supremum norm) is bounded by:
$$||V_{\mathcal{M}_T}^{\pi_\Psi} - V_{\mathcal{M}_T}^{\pi_{S \circ h}}||_\infty \le \frac{2 L_{TV} Q_{max}}{1-\gamma}\epsilon$$
\end{theorem}

\begin{proof}
Let $V_1 := V_{\mathcal{M}_T}^{\pi_\Psi}$ and $V_2 := V_{\mathcal{M}_T}^{\pi_{S \circ h}}$.
Define the Bellman operators $\mathcal{T}_1 := \mathcal{T}_{\mathcal{M}_T}^{\pi_\Psi}$ and $\mathcal{T}_2 := \mathcal{T}_{\mathcal{M}_T}^{\pi_{S \circ h}}$.

Since $V_1$ and $V_2$ are fixed points of their respective Bellman operators:
$$V_1(s) = (\mathcal{T}_1 V_1)(s) \quad \text{and} \quad V_2(s) = (\mathcal{T}_2 V_2)(s)$$

For any arbitrary target state $s \in S_T$, we evaluate the absolute difference and apply the triangle inequality by adding and subtracting $(\mathcal{T}_1 V_2)(s)$:
$$|V_1(s) - V_2(s)| = |(\mathcal{T}_1 V_1)(s) - (\mathcal{T}_2 V_2)(s)|$$
$$\le |(\mathcal{T}_1 V_1)(s) - (\mathcal{T}_1 V_2)(s)| + |(\mathcal{T}_1 V_2)(s) - (\mathcal{T}_2 V_2)(s)|$$

For the first term, because the Bellman operator is a $\gamma$-contraction in the supremum norm, the difference in expectations from state $s$ is strictly bounded by the maximum possible difference across all states:
$$|(\mathcal{T}_1 V_1)(s) - (\mathcal{T}_1 V_2)(s)| \le \gamma ||V_1 - V_2||_\infty$$

For the second term, we evaluate the difference caused by the policy shift:
\begin{align*}
&|(\mathcal{T}_1 V_2)(s) - (\mathcal{T}_2 V_2)(s)| \\
&\quad= \left| \int_A (\pi_S(a|\Psi(s)) - \pi_S(a|h(s))) Q_T^{V_2}(s,a) da \right|
\end{align*}

Applying the absolute value, bounding the action-value by $Q_{max}$ (Assumption 3), and using the definition of Total Variation distance:
$$\le Q_{max} \int_A \left| \pi_S(a|\Psi(s)) - \pi_S(a|h(s)) \right| da$$
$$= 2 Q_{max} D_{TV}(\pi_S(\cdot|\Psi(s)), \pi_S(\cdot|h(s)))$$

Applying the $L_{TV}$-Lipschitz assumption and the state approximation bound $||\Psi(s) - h(s)|| \le \epsilon$:
$$\le 2 Q_{max} L_{TV} ||\Psi(s) - h(s)|| \le 2 L_{TV} Q_{max} \epsilon$$

Substituting both terms back into our original pointwise inequality yields:
$$|V_1(s) - V_2(s)| \le \gamma ||V_1 - V_2||_\infty + 2 L_{TV} Q_{max} \epsilon$$

Since this inequality holds for every state $s \in S_T$, it must also hold for the supremum over all states. Taking the supremum of the left side gives:
$$||V_1 - V_2||_\infty \le \gamma ||V_1 - V_2||_\infty + 2 L_{TV} Q_{max} \epsilon$$

Rearranging the terms completes the proof:
$$(1-\gamma) ||V_1 - V_2||_\infty \le 2 L_{TV} Q_{max} \epsilon$$
$$||V_1 - V_2||_\infty \le \frac{2 L_{TV} Q_{max}}{1-\gamma}\epsilon$$
\end{proof}

\section{Experiments}

In this section, we present the experimental evaluation of ASPECT. We describe the environments, the source and target tasks, and the baselines used to validate our natural language-conditioned imagination framework.

\subsection{Experimental Setup}
We evaluate ASPECT across three distinct environments designed to test complementary aspects of generalization: MiniGrid, MiniWorld, and a custom Fragile Object Manipulation environment. Visualizations of these environments are shown in Figure \ref{fig:envs}. These environments differ in observation modality (image-based vs. feature-based), action space (discrete vs. continuous), reward structure (sparse vs. dense), and the underlying RL algorithm used for the source policy (DQN, PPO, SAC), allowing for a comprehensive analysis of the proposed method's agnosticism. All episodes terminate either upon successful task completion or when the maximum number of timesteps is reached. All experiments were run for 5 random seeds. Implementation details of the RL policies are provided in Appendix \ref{app:rl_details}.

\subsubsection{MiniGrid}
MiniGrid \cite{Minigrid} is a 2D grid-world that provides pixel-based top-down observations. MiniGrid supports discrete action spaces, enabling rapid prototyping of navigation and object-interaction tasks. The source policy is trained using DQN. The environment uses a sparse reward structure, where the agent receives a positive reward only upon successful task completion.

\subsubsection{MiniWorld}
MiniWorld \cite{Miniworld23} is a 3D first-person simulator that provides egocentric, pixel-based visual observations and supports discrete control. The source policy is trained using PPO. Unlike MiniGrid, MiniWorld employs a dense reward setting, where the agent receives incremental shaping rewards for approaching goal objects in addition to a terminal success reward. This setting allows us to test visual generalization and affordance transfer under richer sensory inputs.

\subsubsection{Fragile Object Manipulation Environment}
To explicitly evaluate affordance understanding and force-sensitive control, we developed a custom feature-based environment. The agent interacts with fragile objects, each characterized by a specific \textit{fragility threshold}. If the applied force exceeds this threshold, the object breaks. The source policy is trained using SAC.
The agent operates in a continuous, three-dimensional action space $(\text{turn angle}, \text{move distance}, \text{apply force})$. The reward function is dense: the agent is penalized for breaking an object, rewarded for successful pickups, and receives a positive shaping reward for approaching the target.  

\begin{figure}[htbp]
    \centering
    \begin{subfigure}[b]{0.32\columnwidth}
        \centering
        \includegraphics[width=\textwidth]{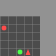}
        \caption{MiniGrid}
        \label{fig:minigrid}
    \end{subfigure}
    \hfill
    \begin{subfigure}[b]{0.32\columnwidth}
        \centering
        \includegraphics[width=\textwidth]{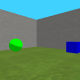}
        \caption{MiniWorld}
        \label{fig:miniworld}
    \end{subfigure}
    \hfill
    \begin{subfigure}[b]{0.32\columnwidth}
        \centering
        \includegraphics[width=\textwidth]{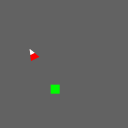}
        \caption{Manipulation}
        \label{fig:manipulation}
    \end{subfigure}
    \caption{Visualizations of the three environments used in our experiments: (a) MiniGrid (2D navigation), (b) MiniWorld (3D egocentric), and (c) Manipulation (continuous control).}
    \label{fig:envs}
\end{figure}

The observation is a 12-dimensional feature vector that encodes the agent's orientation (as $\sin/\cos$ of the heading angle), the object’s relative position and bearing (as $\sin/\cos$, normalized by environment size), one-hot encodings for object \textit{type} (circle/square) and \textit{weight} (light/heavy), and boolean flags indicating whether the object has been picked or broken.

Together, these environments allow us to demonstrate that ASPECT is agnostic to both observation modality (image-based or feature-based), control type (discrete or continuous), and the RL algorithm (DQN, PPO, or SAC), while remaining effective across sparse and dense reward regimes.

\subsection{Source and Target Task Definitions}
For each environment, we train a source policy $\pi_{\text{source}}$ on \textit{source task}, and subsequently evaluate its zero-shot transfer performance on multiple \textit{target tasks} that introduce new objects, altered visual contexts, or both.

\subsubsection{Source Task Settings}
\begin{itemize}
    \item MiniGrid: Pick the \textit{red ball} and avoid the \textit{green ball}.
    \item MiniWorld: Pick the \textit{blue box} and avoid the \textit{green ball}.
    \item Custom Env: Pick both the \textit{light circle} (low force threshold) and \textit{heavy square} (high force threshold) without breaking either.
\end{itemize}

\subsubsection{Target Task Settings (Evaluation)}
We evaluate ASPECT on three categories of generalization challenges, as detailed in Table \ref{tab:target_tasks}:
\begin{enumerate}
    \item \textbf{Case 1 (Unseen Objects)}: Tests whether the agent can reuse a learned skill for unseen objects (e.g., picking a ``yellow duckie'' instead of a ``blue box'') by leveraging affordance-based analogies.
    \item \textbf{Case 2 (Combined Shift)}: Challenges the agent to handle both visual shifts (e.g., texture changes) and semantic shifts (unseen objects) simultaneously.
    \item \textbf{Case 3 (Reversed Task)}: Evaluates robustness to conflicting priors, where the object associated with reward in the source task becomes a distractor in the target task.
\end{enumerate}

\begin{table}[h]
\caption{Target Task Settings (Evaluation). The agent must generalize zero-shot to these novel scenarios.}
\label{tab:target_tasks}
\begin{center}
\begin{small}
\resizebox{\columnwidth}{!}{
\begin{tabular}{@{}llp{5.5cm}@{}}
\toprule
\textbf{Case} & \textbf{Env.} & \textbf{Target Task Description} \\ \midrule
\multirow{3}{*}{\shortstack[l]{\textbf{1. Unseen Objects}\\(Different reward object)}} & MiniGrid & Pick \textit{purple box}, avoid \textit{green ball}. \\
 & MiniWorld & Pick \textit{yellow duckie}, avoid \textit{green ball}. \\
 & Custom & Pick \textit{heavy circle} \& \textit{light square} (inverted weights). \\ \midrule
\multirow{2}{*}{\shortstack[l]{\textbf{2. Combined Shift}\\(Visual + Semantic)}} & MiniGrid & Blue Walls. Pick \textit{purple box}. \\
 & MiniWorld & Wood/Brick textures. Pick \textit{yellow duckie}. \\ \midrule
\multirow{2}{*}{\shortstack[l]{\textbf{3. Reversed Task}\\(Prior as Distractor)}} & MiniGrid & Pick \textit{purple box}, avoid \textit{red ball} (source target). \\
 & MiniWorld & Pick \textit{yellow duckie}, avoid \textit{blue box} (source target). \\ \bottomrule
\end{tabular}
}
\end{small}
\end{center}
\end{table}

\section{Experimental Results}
In this section, we detail the experimental results. 

We conduct a series of experiments designed to evaluate the primary capabilities and advantages of ASPECT. Our evaluation aims to answer two key questions:
\begin{enumerate}
    \item \textbf{Can ASPECT generalize to truly novel tasks?} We test its ability to perform zero-shot knowledge transfer to target tasks that share the same underlying affordance (e.g., "picking") but involve previously unseen objects or different, though semantically similar, observations.
    
    \item \textbf{How data-efficient is ASPECT compared to fine-tuning?} We compare the zero-shot performance of ASPECT against the data efficiency of fine-tuning the source policy on the target task. This highlights the sample complexity advantage of our approach.
\end{enumerate}

\subsection{Baselines}
We compare ASPECT against the following baselines:
\begin{itemize}
    \item \textbf{SF-Simple} \cite{chua2024learning}: A successor feature-based method that learns transferrable representations without complex auxiliary tasks.
    \item \textbf{SF-Reconstruction} \cite{zhang2017}: An approach utilizing successor features combined with a reconstruction auxiliary task to learn robust state representations for navigation across similar environments.
    \item \textbf{PPO/DQN/SAC (Source)}: The original source policy evaluated directly on the target task (zero-shot) to measure the immediate transfer gap.
    \item \textbf{PPO/DQN/SAC (Fine-tuned)}: The source policy fine-tuned on the target task, providing an upper bound on performance (or a strong adaptive baseline) given access to target environment interactions. We denote fine-tuning until convergence as \textbf{FT (Conv.)} and fine-tuning with limited steps as \textbf{FT (Lim.)}.
\end{itemize}

While our approach is conceptually inspired by MAGIK \cite{palattuparambil2025magik}, we do not include it as a baseline. MAGIK is designed to transfer skills between \textit{known} objects using discrete, one-hot class representations. However, ASPECT focuses on generalization to \textit{unseen} objects where such pre-defined one-hot vectors cannot be constructed, rendering MAGIK inapplicable to these experimental settings.

First, we briefly describe our captioning process. To generate noise-free semantic descriptions of the environment, we employed a structured captioning module (see Appendix \ref{app:query_types}). This module populates a predefined template with sensor data, such as object type and location, ensuring accurate and consistent input for the language-conditioned imagination process without requiring complex vision-language model inference or post-processing. For the semantic mapping $M_{LLM}$, we utilized Grok-4.1-fast and Gemini 2.5 Flash as our Large Language Models.

\subsection{Zero-Shot Generalization Results}
We evaluate the zero-shot transfer performance of ASPECT across three generalization scenarios involving novel objects and observational changes.

\subsubsection{Case 1: Generalization to Unseen Objects}
Results for this setting, where the agent transfers its skill to previously unseen objects, are summarized in Table \ref{tab:case1_consolidated}.

\begin{table}[t]
  \caption{Consolidated Experimental Results for Case 1 (Generalization to Unseen Objects). Metrics are reported as counts (out of 10) $\pm$ std dev. "Tgt" = Target, "Dstr" = Distractor, "Lift" = Successful Lifts, "Brk" = Broken Objects.}
  \label{tab:case1_consolidated}
  \centering
  \small
  \setlength{\tabcolsep}{1.5pt}
  \sc
  \resizebox{\columnwidth}{!}{
  \begin{tabular}{lcccccc}
    \toprule
     & \multicolumn{2}{c}{MiniWorld} & \multicolumn{2}{c}{MiniGrid} & \multicolumn{2}{c}{Manip.} \\
     \cmidrule(lr){2-3} \cmidrule(lr){4-5} \cmidrule(lr){6-7}
    Method & Tgt & Dstr & Tgt & Dstr & Lift & Brk \\
    \midrule
    Source & 2.80 $\pm$ 0.80 & 0.20 $\pm$ 0.20 & 0.00 $\pm$ 0.00 & 0.00 $\pm$ 0.00 & 2.80 $\pm$ 1.06 & 4.60 $\pm$ 0.04 \\
    FT (Conv.) & 10.0 $\pm$ 0.00 & 0.60 $\pm$ 0.40 & 9.60 $\pm$ 0.24 & 0.00 $\pm$ 0.00 & 9.80 $\pm$ 0.20 & 0.00 $\pm$ 0.00 \\
    FT (Lim.) & 8.00 $\pm$ 1.04 & 0.40 $\pm$ 0.24 & 7.20 $\pm$ 0.86 & 0.40 $\pm$ 0.40 & 6.60 $\pm$ 0.87 & 2.40 $\pm$ 1.51 \\
    SF-Simp. & 0.20 $\pm$ 0.20 & 0.00 $\pm$ 0.00 & 5.60 $\pm$ 1.28 & 0.00 $\pm$ 0.00 & - & - \\
    SF-Rec. & 0.00 $\pm$ 0.00 & 0.00 $\pm$ 0.00 & 9.00 $\pm$ 0.20 & 0.00 $\pm$ 0.00 & - & - \\
    ASPECT & \textbf{8.40 $\pm$ 0.24} & 0.02 $\pm$ 0.02 & \textbf{9.40 $\pm$ 0.40} & 0.00 $\pm$ 0.00 & \textbf{9.60 $\pm$ 0.24} & 0.40 $\pm$ 0.24 \\
    \bottomrule
  \end{tabular}
  }
\end{table}

In the MiniWorld environment (Table \ref{tab:case1_consolidated}), where the target object is a ``yellow duckie'' (unseen during training), the standard PPO policy fails completely (2.80 success), as it relies on specific visual features of the source object (``blue box''). Similarly, even though the successor feature baselines (SF-Simple and SF-Reconstruction) are allowed to interact with the target environment, they fail to generalize (see Figure \ref{fig:sf_miniworld} in Appendix \ref{app:additional_curves}). In contrast, ASPECT, which operates zero-shot without any target interaction, achieves a success rate of 8.40 $\pm$ 0.24, outperforming the PPO policy fine-tuned for 20K steps (8.00 $\pm$ 1.04). While it strictly underperforms the fully converged fine-tuned PPO upper bound (10.00 $\pm$ 0.00), it achieves this without any gradient updates on the target task. This demonstrates the effectiveness of the LLM-guided mapping in bridging the semantic gap between ``blue box'' and ``yellow duckie''.

In the MiniGrid environment (Table \ref{tab:case1_consolidated}), ASPECT again demonstrates rigorous zero-shot performance (9.40 $\pm$ 0.40), significantly outperforming the DQN policy fine-tuned for 50K steps (7.20 $\pm$ 0.86) and effectively matching the fully converged fine-tuned DQN baseline (9.60 $\pm$ 0.24). While SF-Reconstruction performs better here (9.00 $\pm$ 0.20) than in MiniWorld, potentially due to the simpler visual grid structure, it still lags behind ASPECT, despite having the advantage of environmental interaction. The baseline DQN policy completely fails to transfer.

Finally, in the Fragile Object Manipulation task (Table \ref{tab:case1_consolidated}), the challenge involves inverting physical properties: picking a ``heavy circle'' and ``light square'' when trained on the opposite. The standard SAC policy struggles significantly, lifting only 2.80 $\pm$ 1.06 objects on average, and incurs a high failure rate, breaking an average of 4.60 objects (as shown in the ``Num Objects broken'' column). ASPECT successfully navigates this affordance inversion, lifting 9.60 $\pm$ 0.24 objects, which is comparable to the fine-tuned SAC expert (9.80 $\pm$ 0.20) and significantly outperforms the SAC policy fine-tuned for 10K steps (6.60 $\pm$ 0.87), by correctly mapping the target objects to their source counterparts based on the ``fragility'' interactions described in text.

\subsubsection{Case 2: Combined Generalization}
Case 2 introduces a more difficult challenge: generalising to tasks that involve both a novel object and a shift in environmental observations (e.g., room colour or texture, See Figure \ref{fig:imagination_viz} for more variations). The results are presented in Table \ref{tab:case2_consolidated}.

In MiniWorld (Table \ref{tab:case2_consolidated}), the agent faces a scene with a wooden floor and brick walls (vs. grass/concrete in source) and must pick a ``yellow duckie''. The standard PPO baseline struggles significantly (3.20 $\pm$ 0.66). However, ASPECT achieves a success rate of 8.80 $\pm$ 0.37, outperforming the PPO policy fine-tuned for 20K steps (7.60 $\pm$ 0.50) and approaching the fully converged fine-tuned PPO baseline (9.20 $\pm$ 0.37).

\begin{table}[htbp]
  \caption{Consolidated Experimental Results for Case 2 (Combined Generalization). Metrics are reported as counts (out of 10) $\pm$ std dev. "Tgt" = Target, "Dstr" = Distractor.}
  \label{tab:case2_consolidated}
  \centering
  \small
  \setlength{\tabcolsep}{1.5pt}
  \sc
  \resizebox{\columnwidth}{!}{
  \begin{tabular}{lcccc}
    \toprule
     & \multicolumn{2}{c}{MiniWorld} & \multicolumn{2}{c}{MiniGrid} \\
     \cmidrule(lr){2-3} \cmidrule(lr){4-5}
    Method & Tgt & Dstr & Tgt & Dstr \\
    \midrule
    Source & 3.20 $\pm$ 0.66 & 0.20 $\pm$ 0.20 & 0.00 $\pm$ 0.00 & 0.00 $\pm$ 0.00 \\
    FT (Conv.) & 9.20 $\pm$ 0.37 & 0.40 $\pm$ 0.24 & 9.80 $\pm$ 0.20 & 0.00 $\pm$ 0.00 \\
    FT (Lim.) & 7.60 $\pm$ 0.50 & 0.40 $\pm$ 0.24 & 6.20 $\pm$ 0.73 & 0.60 $\pm$ 0.24 \\
    SF-Simp. & 0.00 $\pm$ 0.00 & 0.00 $\pm$ 0.00 & 4.80 $\pm$ 1.62 & 0.00 $\pm$ 0.00 \\
    SF-Rec. & 0.00 $\pm$ 0.00 & 0.00 $\pm$ 0.00 & 4.60 $\pm$ 2.09 & 0.00 $\pm$ 0.00 \\
    ASPECT & \textbf{8.80 $\pm$ 0.37} & 0.04 $\pm$ 0.24 & \textbf{8.80 $\pm$ 0.58} & 0.40 $\pm$ 0.24 \\
    \bottomrule
  \end{tabular}
  }
\end{table}

Similarly, in the MiniGrid environment (Table \ref{tab:case2_consolidated}), where the wall colour changes to blue, ASPECT maintains high performance (8.80 $\pm$ 0.58). It significantly outperforms the DQN policy fine-tuned for 50K steps (6.20 $\pm$ 0.73) and closely trails the fully converged fine-tuned DQN upper bound (9.80 $\pm$ 0.20). This result highlights the robustness of our text-conditioned imagination to compound distribution shifts. Qualitative visualizations of this process are shown in Appendix \ref{sec:unseen_setting}.

\subsubsection{Case 3: Unseen Object and Reversed Task}
In this scenario, we evaluate the agent's robustness to conflicting priors. The object that was rewarding in the source task is present in the target task but is now a distractor (or non-rewarding), while a novel object is the target.

In MiniWorld (Table \ref{tab:case3_consolidated}), the source rewarding object (``blue box'') acts as a distractor. The standard PPO policy exhibits a strong bias towards its training prior, mistakenly picking the ``blue box'' 8.20 $\pm$ 0.37 times, while rarely picking the correct target (``yellow duckie'', 2.80 $\pm$ 0.66). In contrast, ASPECT overcomes this bias, achieving 8.40 $\pm$ 0.54 success on the novel target, which is significantly higher than the PPO policy fine-tuned for 20K steps (5.60 $\pm$ 1.36). This is because the LLM hallucinate the rewarding object (``blue box'') as a distractor (``green ball'') in the source environment, and the agent ignores it.

Results in MiniGrid (Table \ref{tab:case3_consolidated}) follow a similar pattern. The DQN baseline is completely fixated on the ``red ball'' (source reward), picking it 7.20 $\pm$ 0.86 times and never picking the correct ``purple box''. ASPECT successfully generalizes, picking the ``purple box'' 9.00 $\pm$ 0.54 times, dramatically outperforming the DQN policy fine-tuned for 50K steps (3.40 $\pm$ 0.87). These results demonstrate that the LLM guides the hallucination of the target environment according to the current and the known task, allowing the agent to effectively filter out obsolete reward signals. 

\begin{table}[t]
  \caption{Consolidated Experimental Results for Case 3 (Reverse Task). Metrics are reported as counts (out of 10) $\pm$ std dev. "Old Target" refers to the object rewarding in source but distracting in target.}
  \label{tab:case3_consolidated}
  \begin{center}
    \begin{small}
      \begin{sc}
        \resizebox{\columnwidth}{!}{
        \begin{tabular}{lcccc}
          \toprule
           & \multicolumn{2}{c}{MiniWorld} & \multicolumn{2}{c}{MiniGrid} \\
           \cmidrule(lr){2-3} \cmidrule(lr){4-5}
          Method & Target & Old Target & Target & Old Target \\
          \midrule
          Source & 2.80 $\pm$ 0.66 & 8.20 $\pm$ 0.37 & 0.00 $\pm$ 0.00 & 7.20 $\pm$ 0.86 \\
          FT (Conv.) & 9.80 $\pm$ 0.24 & 0.00 $\pm$ 0.00 & 9.80 $\pm$ 0.20 & 0.00 $\pm$ 0.00 \\
          FT (Lim.) & 5.60 $\pm$ 1.36 & 0.40 $\pm$ 0.24 & 3.40 $\pm$ 0.87 & 0.80 $\pm$ 0.37 \\
          SF-Simple & 0.00 $\pm$ 0.00 & 0.00 $\pm$ 0.00 & 3.80 $\pm$ 1.68 & 0.00 $\pm$ 0.00 \\
          SF-Reconstruction & 0.00 $\pm$ 0.00 & 0.00 $\pm$ 0.00 & 1.20 $\pm$ 0.96 & 0.00 $\pm$ 0.00 \\
          ASPECT & \textbf{8.40 $\pm$ 0.54} & 0.20 $\pm$ 0.20 & \textbf{9.00 $\pm$ 0.54} & 0.40 $\pm$ 0.24 \\
          \bottomrule
        \end{tabular}
        }
      \end{sc}
    \end{small}
  \end{center}
\end{table}



\subsection{Data Efficiency and Sample Complexity}
\label{sec:data_efficiency}

A key advantage of our approach is its data efficiency. By leveraging the semantic priors of the LLM and the structural priors of the VAE, ASPECT enables immediate transfer without the need for environment interaction in the target domain.

To quantify this ``zero-shot gap,'' we compare ASPECT's performance against the fine-tuning curves of standard RL baselines. In the MiniWorld environment, ASPECT outperforms the PPO baseline fine-tuned for 20K steps across all cases (Cases 1, 2, and 3) (see Appendix \ref{app:additional_results}, Figure \ref{fig:miniworld_curves}). Similarly, in the MiniGrid environment, ASPECT consistently outperforms the DQN baseline fine-tuned for 50K steps across all cases (see Appendix \ref{app:additional_curves}, Figure \ref{fig:minigrid_curves}). Furthermore, in the Fragile Object Manipulation environment, ASPECT outperforms the SAC baseline fine-tuned for 10K steps (see Appendix \ref{app:additional_curves}, Figure \ref{fig:manipulation_curve}).

This represents a significant saving in sample complexity, which is particularly critical for real-world applications where data collection is expensive or dangerous. While fine-tuning eventually yields slightly higher asymptotic performance (e.g., reaching 10.0 success in MiniWorld), ASPECT provides a ``jump-start'' that equates to tens of thousands of training steps, effectively bypassing the initial exploration and adaptation phase.

\subsection{Robustness to Unstructured VLM Captions}

While structured captions ensure consistency, they rely on predefined templates that may not scale to open-ended scenarios. To evaluate ASPECT's robustness to variable and unstructured language, we conducted experiments on the MiniWorld environment using captions generated by a Vision-Language Model (VLM).

We used the \texttt{nvidia/\allowbreak nemotron-nano-12b-v2-vl} model to generate captions from observation frames, prompting it with both the image and auxiliary sensor data (object type and location). Due to the low resolution of the observations ($80 \times 80$), the VLM occasionally produced inaccuracies. To mitigate this, we implemented a verification step where captions were checked for completeness; inaccurate instances were re-generated using stronger models (Grok-4.1 and Gemini 2.5 Flash) or post-processed to remove hallucinations. The resulting captions were significantly more variable in length and structure compared to the template-based approach. To handle these longer, unstructured descriptions (see Appendix \ref{app:query_types} for examples), we replaced the standard CLIP text encoder with LongCLIP \cite{zhang2024long}, which supports inputs exceeding the 77-token limit of standard CLIP.

Table \ref{tab:miniworld_case_noise} (Appendix \ref{app:additional_results}) presents the results of ASPECT using these noisy, unstructured captions across all three generalization cases in MiniWorld. Remarkably, the method maintains high performance, achieving success rates comparable to those obtained with clean, structured captions (e.g., $8.60 \pm 0.50$ for Case 1 vs. $8.40 \pm 0.24$ with structured text). In all cases, the agent successfully identifies and picks the rewarding object while completely avoiding the distractor ($0.00$ failure rate). These results underscore ASPECT's ability to extract relevant semantic cues even from noisy, variable-length natural language descriptions, further validating the flexibility of the text-conditioned imagination framework. Crucially, this capability suggests that our method can be applied to real-world settings where structured captions are unavailable, a constraint that would otherwise severely limit real-world applicability.

\section{Discussion and Limitations}

In this work, we introduced ASPECT, a novel framework for zero-shot policy transfer that leverages the semantic reasoning of Large Language Models and the generative capabilities of text-conditioned VAEs. By treating the LLM as a dynamic semantic operator, our approach enables agents to ``imagine'' and solve analogous target tasks by relating them to prior experiences. Our experiments across diverse environments ranging from grid worlds to continuous manipulation demonstrated that ASPECT can robustly generalize to unseen objects, visual shifts, and even contradictory reward structures without any training in the target domain. Furthermore, we showed that this method remains effective even when relying on noisy, unstructured captions from vision-language models, highlighting its potential for real-world applications.

However, our approach has limitations. First, reliance on the generative capabilities of the VAE means the system is susceptible to imagination artifacts. As detailed in Appendix \ref{app:failure_cases}, the model can struggle with extreme close-ups or fail to materialize remapped objects in the imagined scene, potentially leading to policy failure. Second, by leveraging Large Language Models (LLMs) as semantic operators, our system inherits the known limitations of these models, including hallucinations, biases, and unpredictability. An incorrect semantic mapping generated by the LLM could lead to agent behaviors that are misaligned with the intended user goals, posing safety risks in critical applications. Future work must address robust verification and safety constraints to mitigate these risks before deployment in sensitive domains.

\section*{Impact Statement}

This paper presents work whose goal is to advance the field of Reinforcement Learning, specifically focusing on zero-shot generalization through language-conditioned imagination. Our method, ASPECT, demonstrates the potential to create more adaptable agents capable of operating in diverse and novel analogous environments without extensive retraining. This has positive implications for the scalability of autonomous systems in real-world settings, such as robotics and personalized assistants.

\bibliography{example_paper}

@inproceedings{palattuparambil2025magik,
  title     = {MAGIK: Mapping to Analogous Goals via Imagination-enabled Knowledge Transfer},
  author    = {Palattuparambil, Ajsal Shereef and Karimpanal, Thommen George and Rana, Santu},
  booktitle = {Proceedings of the 28rd European Conference on Artificial Intelligence (ECAI 2025)},
  publisher = {IOS Press},
  year      = {2025},
  pages     = {2874--2881},
  location  = {Bologna, Italy},
  series    = {Frontiers in Artificial Intelligence and Applications},
  volume    = {413},
  doi       = {10.3233/FAIA413}
}

@article{zhang2018dissection,
  title={A dissection of overfitting and generalization in continuous reinforcement learning},
  author={Zhang, Amy and Ballas, Nicolas and Pineau, Joelle},
  journal={arXiv preprint arXiv:1806.07937},
  year={2018}
}

@article{hafner2019dream,
  title={Dream to control: Learning behaviors by latent imagination},
  author={Hafner, Danijar and Lillicrap, Timothy and Ba, Jimmy and Norouzi, Mohammad},
  journal={arXiv preprint arXiv:1912.01603},
  year={2019}
}

@article{nair2018visual,
  title={Visual reinforcement learning with imagined goals},
  author={Nair, Ashvin V and Pong, Vitchyr and Dalal, Murtaza and Bahl, Shikhar and Lin, Steven and Levine, Sergey},
  journal={Advances in neural information processing systems},
  volume={31},
  year={2018}
}

@article{schulman2017proximal,
  title={Proximal policy optimization algorithms},
  author={Schulman, John and Wolski, Filip and Dhariwal, Prafulla and Radford, Alec and Klimov, Oleg},
  journal={arXiv preprint arXiv:1707.06347},
  year={2017}
}

@article{mnih2015human,
  title={Human-level control through deep reinforcement learning},
  author={Mnih, Volodymyr and Kavukcuoglu, Koray and Silver, David and Rusu, Andrei A and Veness, Joel and Bellemare, Marc G and Graves, Alex and Riedmiller, Martin and Fidjeland, Andreas K and Ostrovski, Georg and others},
  journal={nature},
  volume={518},
  number={7540},
  pages={529--533},
  year={2015},
  publisher={Nature Publishing Group}
}

@inproceedings{haarnoja2018soft,
  title={Soft actor-critic: Off-policy maximum entropy deep reinforcement learning with a stochastic actor},
  author={Haarnoja, Tuomas and Zhou, Aurick and Abbeel, Pieter and Levine, Sergey},
  booktitle={International conference on machine learning},
  pages={1861--1870},
  year={2018},
  organization={Pmlr}
}

@article{kingma2013auto,
  title={Auto-encoding variational bayes},
  author={Kingma, Diederik P and Welling, Max},
  journal={arXiv preprint arXiv:1312.6114},
  year={2013}
}

@inproceedings{radford2021learning,
  title={Learning transferable visual models from natural language supervision},
  author={Radford, Alec and Kim, Jong Wook and Hallacy, Chris and Ramesh, Aditya and Goh, Gabriel and Agarwal, Sandhini and Sastry, Girish and Askell, Amanda and Mishkin, Pamela and Clark, Jack and others},
  booktitle={International conference on machine learning},
  pages={8748--8763},
  year={2021},
  organization={PmLR}
}

@inproceedings{perez2018film,
  title={Film: Visual reasoning with a general conditioning layer},
  author={Perez, Ethan and Strub, Florian and De Vries, Harm and Dumoulin, Vincent and Courville, Aaron},
  booktitle={Proceedings of the AAAI conference on artificial intelligence},
  volume={32},
  year={2018}
}

@article{vaswani2017attention,
  title={Attention is all you need},
  author={Vaswani, Ashish and Shazeer, Noam and Parmar, Niki and Uszkoreit, Jakob and Jones, Llion and Gomez, Aidan N and Kaiser, {\L}ukasz and Polosukhin, Illia},
  journal={Advances in neural information processing systems},
  volume={30},
  year={2017}
}

@article{brown2020language,
  title={Language models are few-shot learners},
  author={Brown, Tom and Mann, Benjamin and Ryder, Nick and Subbiah, Melanie and Kaplan, Jared D and Dhariwal, Prafulla and Neelakantan, Arvind and Shyam, Pranav and Sastry, Girish and Askell, Amanda and others},
  journal={Advances in neural information processing systems},
  volume={33},
  pages={1877--1901},
  year={2020}
}

@inproceedings{higgins2017beta,
  title={beta-vae: Learning basic visual concepts with a constrained variational framework},
  author={Higgins, Irina and Matthey, Loic and Pal, Arka and Burgess, Christopher and Glorot, Xavier and Botvinick, Matthew and Mohamed, Shakir and Lerchner, Alexander},
  booktitle={International conference on learning representations},
  year={2017}
}

@article{ganin2016domain,
  title={Domain-adversarial training of neural networks},
  author={Ganin, Yaroslav and Ustinova, Evgeniya and Ajakan, Hana and Germain, Pascal and Larochelle, Hugo and Laviolette, Fran{\c{c}}ois and March, Mario and Lempitsky, Victor},
  journal={Journal of machine learning research},
  volume={17},
  number={59},
  pages={1--35},
  year={2016}
}

@inproceedings{gamrian2019transfer,
  title={Transfer learning for related reinforcement learning tasks via image-to-image translation},
  author={Gamrian, Shani and Goldberg, Yoav},
  booktitle={International conference on machine learning},
  pages={2063--2072},
  year={2019},
  organization={PMLR}
}

@article{zambaldi2018relational,
  title={Relational deep reinforcement learning},
  author={Zambaldi, Vinicius and Raposo, David and Santoro, Adam and Bapst, Victor and Li, Yujia and Babuschkin, Igor and Tuyls, Karl and Reichert, David and Lillicrap, Timothy and Lockhart, Edward and others},
  journal={arXiv preprint arXiv:1806.01830},
  year={2018}
}

@article{barreto2017successor,
  title={Successor features for transfer in reinforcement learning},
  author={Barreto, Andr{\'e} and Dabney, Will and Munos, R{\'e}mi and Hunt, Jonathan J and Schaul, Tom and van Hasselt, Hado P and Silver, David},
  journal={Advances in neural information processing systems},
  volume={30},
  year={2017}
}

@article{chua2024learning,
  title={Learning successor features the simple way},
  author={Chua, Raymond and Ghosh, Arna and Kaplanis, Christos and Richards, Blake A and Precup, Doina},
  journal={Advances in Neural Information Processing Systems},
  volume={37},
  pages={49957--50030},
  year={2024}
}

@inproceedings{Minigrid,
  author    = {Maxime Chevalier{-}Boisvert and Bolun Dai and Mark Towers and Rodrigo Perez{-}Vicente and Lucas Willems and Salem Lahlou and Suman Pal and Pablo Samuel Castro and Jordan Terry},
  title     = {Minigrid {\&} Miniworld: Modular {\&} Customizable Reinforcement Learning Environments for Goal-Oriented Tasks},
  booktitle = {Advances in Neural Information Processing Systems 36, New Orleans, LA, USA},
  month     = {December},
  year      = {2023}
}

@article{Miniworld23,
  author  = {Maxime Chevalier-Boisvert and Bolun Dai and Mark Towers and Rodrigo de Lazcano and Lucas Willems and Salem Lahlou and Suman Pal and Pablo Samuel Castro and Jordan Terry},
  title   = {Minigrid \& Miniworld: Modular \& Customizable Reinforcement Learning Environments for Goal-Oriented Tasks},
  journal = {CoRR},
  volume  = {abs/2306.13831},
  year    = {2023}
}

@inproceedings{zhang2017,
  title        = {Deep reinforcement learning with successor features for navigation across similar environments},
  author       = {Zhang, Jingwei and Springenberg, Jost Tobias and Boedecker, Joschka and Burgard, Wolfram},
  booktitle    = {2017 IEEE/RSJ International Conference on Intelligent Robots and Systems (IROS)},
  pages        = {2371--2378},
  year         = {2017},
  organization = {IEEE}
}

@inproceedings{zhang2024long,
  title={Long-clip: Unlocking the long-text capability of clip},
  author={Zhang, Beichen and Zhang, Pan and Dong, Xiaoyi and Zang, Yuhang and Wang, Jiaqi},
  booktitle={European conference on computer vision},
  pages={310--325},
  year={2024},
  organization={Springer}
}

@article{stable-baselines3,
  author  = {Antonin Raffin and Ashley Hill and Adam Gleave and Anssi Kanervisto and Maximilian Ernestus and Noah Dormann},
  title   = {Stable-Baselines3: Reliable Reinforcement Learning Implementations},
  journal = {Journal of Machine Learning Research},
  year    = {2021},
  volume  = {22},
  number  = {268},
  pages   = {1-8},
  url     = {http://jmlr.org/papers/v22/20-1364.html}
}
\bibliographystyle{icml2026}

\newpage
\appendix
\onecolumn

\section{Algorithm Pseudocode}
\label{app:algorithm}
\begin{algorithm}[h!]
\caption{Zero-Shot Action Selection via LLM-Conditioned Imagination}
\label{alg:llm_magik}
\begin{algorithmic}
\STATE {\bfseries Input:} Target observation $s_t$, VLM $V(\cdot)$, LLM $M_{LLM}(\cdot, \text{context})$, Text Encoder $E_{text}(\cdot)$, VAE Encoder $q_{\phi}(z|x)$, VAE Decoder $p_{\theta}(x|z, e_d)$, Pre-trained source policy $\pi_{source}(a|s)$
\STATE {\bfseries Output:} Action $a$ to execute in the target environment

\STATE // Step 1 \& 2: Observe and Describe Target State
\STATE $c_t \gets V(s_t)$

\STATE // Step 3: Semantic Manipulation via LLM
\STATE $c_{t \to s} \gets M_{LLM}(c_t, \text{context})$ \COMMENT{Get source-aligned description}

\STATE // Step 4: Encode Semantic-Orthogonal Features
\STATE $z \sim q_{\phi}(z|s_t)$

\STATE // Step 5: Generate Imagined State
\STATE $e_{c_{t \to s}} \gets E_{text}(c_{t \to s})$ \COMMENT{Embed source caption}
\STATE $s_{imagined} \sim p_{\theta}(x|z, e_{c_{t \to s}})$ \COMMENT{Reconstruct with $z$ and source text embedding}

\STATE // Step 6: Execute Source Policy
\STATE $a \sim \pi_{source}(a | s_{imagined})$

\STATE \textbf{return} $a$
\end{algorithmic}
\end{algorithm}

\section{LLM \& VLM Prompts}
\label{app:llm_details}

In this section, we provide the specific prompts and context used for the captioning module and the LLM semantic operator $M_{LLM}$.

\subsection{System Prompt: Observation Captioning}
\begin{verbatim}
You are a high-precision computer vision annotator for the MiniWorld 3D environment
for RL research. Your task is to convert low-resolution visual observations into
structured, spatially accurate text descriptions.

### 1. OBJECT IDENTIFICATION
The environment contains interactable objects of various types and colors.
* **Standard Objects:** Common examples include "Blue Box," "Green Ball," "Yellow
  Duckie," "Blue Box," "Medkit," and "Key."
* **General Rule:** If you see an object that is not in the list above, describe
  it using **[Color] + [Shape]** (e.g., "Purple Cylinder," "White Cube").
* **Multiple Objects:** If multiple objects are present, list them all (e.g., "A
  Blue Box and a Green Ball").

### 2. ENVIRONMENT & SKY
* **Walls & Floor:** Describe the specific color and texture visible (e.g., "Gray
  concrete wall," "Green grass floor").
* **Sky:** Include the phrase **"Blue sky"** ONLY if the sky is visible. If the
  wall obscures the sky completely (e.g., close-up view), do not mention the sky.

### 3. SENSOR DATA INTEGRATION
You will be provided with "Additional Sensor Data (Ground Truth)".
* **MANDATORY:** You MUST include the exact **distance** and **angle** numeric
  values provided in the sensor data for every visible object.
* **Direction:** Explicitly state the direction (e.g., "to the left", "to the
  right") as given.

**Constraints:**
* **NO META-COMMENTARY:** NEVER use the phrase "sensor data" or "indicated by
  sensor data" in the output caption. The sensor data is context for YOU, not text
  to be quoted.
* **Length Limit:** Keep captions concise and small. For empty scenes, use a
  sentence describing the environment and state no objects are visible.
* If no objects are visible, do NOT justify why. simply state the layout of the
  scene (Wall, floor and sky if available) and mention that there is no objects.
* All the image will contain a wall and the floor. Caption should contain a
  description of the wall and the floor. Remember caption should be small and concise.
* Give extra emphasis on the objects in the sensor data. ensure they are properly
  described and the distance and angle are correctly mentioned.
* DO NOT use any starting phrases such as "The image depicts/shows".
\end{verbatim}

\subsection{System Prompt: Semantic Alignment (Source Description)}
We used the same system prompt across all environments. The system prompt for $M_{LLM}$ is provided below:

\begin{verbatim}
You are an imagination reasoning assistant that simulates how a reinforcement learning
(RL) agent mentally hallucinates its observation to reuse known skills and solve new
target tasks.

The user prompt will contain the following:

\end{verbatim}
\vspace{-1em}
\noindent
$\left.
\begin{array}{@{}l@{}}
\text{\texttt{1. A brief about the environment.}} \\
\text{\texttt{2. What agent knows.}} \\
\text{\texttt{3. What is the target task?}} \\
\text{\texttt{4. A description of the current observation}}
\end{array}
\right\} \text{This is the Context, } \mathcal{C}$
\vspace{-0.5em}
\begin{verbatim}

###ROLE AND PURPOSE

You must:
1. Show detailed reasoning — step-by-step interpretation of how imagination occurs.
2. Output a final JSON containing the transformed or unchanged scene description.

You are not just answering but *solving the target task through reasoning* —
explain what changes are needed, why, and finally give the exact JSON result.

###HOW IMAGINATION WORKS

- The agent performs strictly its known skills (source tasks).
- When the target task differs (Agent absolutely can't perform the target task
  even if there is a subtle differences such as object type/color or background.
  Give extra emphasis on the differences between source and target and map the
  target to source wherever required), the agent imagines — mentally alters its
  perception of the scene so the new goal is solvable using the same skill.
- If you determine that there are changes between the target and source and
  imagination would be necessary, make the changed description as close as
  possible to the source task.
- Imagination occurs strictly in **observation space**, not the physical world.

1. **Minimal Transformation**
   - Modify only what is necessary to make the target solvable.
   - Preserve spatial layout, geometry, and environment details.

2. **Affordance Reasoning**
   - If two objects afford the same action (e.g., pick, push, open), they can be substituted.
   - Example:
     - Known: "pick red ball"
     - Target: "pick green ball"
     → Imagine the green ball as red. This will direct the agent to pick the
       red ball, as it is a known skill. In reality, it is picking the green
       ball.
   - If the target task requires dealing with objects that the agent doesn't
     know, the agent must *imagine* that object as the closest recognizable
     object it knows provided affordance are matching.

3. **Multi-Object or Sequential Tasks**
   - Give proper attention to the spatial position if there are multiple object.
     Don't mix up the positions.

4. **No Fabrication**
   - Never add or invent objects or properties not present in the input scene.

5. **Realism and Consistency**
   - Maintain the original tone, structure, style, and spatial wording. Do not
     invent additional texts.
   - Modify only essential object properties (color, shape, size, etc.).


###REASONING STYLE

- Think step-by-step, like a human reasoning through perception.
- Explain:
  1. What the agent currently knows.
  2. What the target task is.
  3. What are the differences between the target and the source task.
  4. How to map such differences to the source task.
  5. Whether the subtask division is necessary.
  6. What is visible in the current scene?
  7. What minimal changes are needed and why?
   8. Can the minimal changes help the agent to solve the target task by
      mentally imagining an altered scene?. Check against each of the subtasks.
      Discard if the change doesn't affect the subtask.
- Use clear, causal reasoning before outputting the JSON.
- Stop reasoning once the decision is made.


###OUTPUT FORMAT

After reasoning, always output only valid JSON in this format:

{
  "imagine": true | false,
  "description": "<rewritten or unchanged scene description>"
}

- "imagine": true → if imagination was applied to alter the scene.
- "imagine": false → if no change was necessary or possible.
- "description": the final complete and realistic scene.
- Do **not** include extra commentary, code, or markdown after the JSON.
\end{verbatim}

\subsection{Context Construction}
The context $\mathcal{C}$ provided to the LLM consists of four main components:
\begin{enumerate}
    \item A brief description of the environment dynamics and rules.
    \item The source task description (what the agent knows).
    \item The target task description.
    \item A description of the current observation.
\end{enumerate}

\textbf{Example (MiniWorld):}
\begin{itemize}
    \item \textbf{Environment Context:}
    \begin{verbatim}
- The agent operates in a partially observable 3D gridworld-like room.
  The agent sees a portion of the room.
- At the start of each episode, the agent and objects are randomly
  initialised in the environment.
- The agent can perform the following actions: rotate left/right,
  move forward/backward, and pick up objects.
- The environment may contains different objects of different color.
- The agent task is to pick/avoid the objects according to the mission string.
- Since, the observation is partial, the agent can explore the environment
  by moving around to find the objects to pick.
- Once one object is picked, the object dissapears from the scene and it is
  added to agent's inventory, which it can hold forever. This doesn't
  prevent picking another object later.
- The agent can store multiple objects in it's inventory.
- Non-interactive elements (walls, floor, background) cannot be acted upon.
- The agent receives a reward upon successfully completing the Target task
  (for example, picking the specified object from specified room).
- Each episode ends once the Target task is completed or a maximum
  step limit is reached.
    \end{verbatim}
    \item \textbf{What Agent Knows:} ``Pick the blue box and avoid green ball from the room with grass floor and concrete wall''.
    \item \textbf{Target Task:} ``Pick yellow duckie and avoid green ball from the room with grass floor and concrete wall''.
    \item \textbf{Current Observation Description:} ``A yellow duckie is visible at 1.0 units and 34.3 degrees to the left, and a green ball is visible at 3.1 units and 32.9 degrees to the left, both located on a green grass floor surrounded by grey walls under a blue sky.''
\end{itemize}
\subsection{Query Examples}
Below are examples of how the LLM transforms target observation descriptions into source-aligned descriptions to enable zero-shot transfer. The reasoning shown is a summary of the LLM's internal thought process.

\textbf{Example 1: Unseen Object Adaptation}
\begin{itemize}
    \item \textbf{Input:} ``A yellow duckie is visible at 1.0 units and 34.3 degrees to the left, and a green ball is visible at 3.1 units and 32.9 degrees to the left, both located on a green grass floor surrounded by grey walls under a blue sky.''
    \item \textbf{Output:} ``A \textbf{blue box} is visible at 1.0 units and 34.3 degrees to the left, and a green ball is visible at 3.1 units and 32.9 degrees to the left, both located on a green grass floor surrounded by grey walls under a blue sky.''
    \item \textit{Reasoning Step:} The agent reimagines the target object (yellow duckie) as the source reward object (blue box) to trigger the correct pickup policy. The environment context remains unchanged as it matches the source.
\end{itemize}

\textbf{Example 2: Distractor Persistence}
\begin{itemize}
    \setlength{\itemsep}{0pt}
    \setlength{\parskip}{0pt}
    \item \textbf{Input:} ``A green ball is visible at a distance of 4.2 units and an angle of 24.6 degrees to the left, located on a green grass floor surrounded by gray walls under a blue sky.''
    \item \textbf{Output:} ``A green ball is visible at a distance of 4.2 units and an angle of 24.6 degrees to the left, located on a green grass floor surrounded by gray walls under a blue sky.'' (No Change)
    \item \textit{Reasoning Step:} No semantic change is needed because the detected object (green ball) matches the source distractor, which the agent is already trained to avoid.
\end{itemize}

\textbf{Example 3: Complex Environmental \& Semantic Shift}
\begin{itemize}
    \setlength{\itemsep}{0pt}
    \setlength{\parskip}{0pt}
    \item \textbf{Input:} ``A yellow duckie is visible at 1.0 units and 34.3 degrees to the left, and a blue box is visible at 3.1 units and 32.9 degrees to the left, both located on a wodden floor surrounded by brick walls under a blue sky.''
    \item \textbf{Output:} ``A \textbf{blue box} is visible at 1.0 units and 34.3 degrees to the left, and a \textbf{green ball} is visible at 3.1 units and 32.9 degrees to the left, both located on a \textbf{green grass floor} surrounded by \textbf{grey walls} under a blue sky.''
    \item \textbf{Reasoning Step:} This complex mapping requires three simultaneous adjustments:
    \begin{enumerate}
        \item Remapping the target reward (yellow duckie) to the source reward (\textbf{blue box}).
        \item Remapping the old source reward (blue box), which is now a distractor, to a known distractor (\textbf{green ball}) to ensure avoidance.
        \item Hallucinating the environment textures (wooden floor/brick walls) back to the training environment (\textbf{grass floor/concrete walls}) to ensure feature consistency.
    \end{enumerate}
\end{itemize}

\section{Structured vs. Unstructured Captions}
\label{app:query_types}

\subsection{Structured Captions}
The structured query is template filling. For example, in the MiniWorld environment:
\begin{quote}
The agent is in a room with a grass floor $<$floor description$>$ and concrete walls $<$wall description$>$. A blue box $<$object1$>$ is found to the left at angle 10 at a distance of 3.5 units $<$location1$>$. A green ball $<$object2$>$ is found to the right at angle 22.7 at a distance of 2.5 units $<$location2$>$.
\end{quote}
If no objects are visible:
\begin{quote}
The agent is in a room with grass floor and concrete walls. No objects are visible in the current view.
\end{quote}

\subsection{Unstructured Captions}
The unstructured query varies significantly. Here are a few examples:
\begin{itemize}
    \item ``A green ball is visible at a distance of 3.9 units and an angle of 34.2 degrees to the left. It is located on a green grass floor surrounded by gray walls under a blue sky.''
    \item ``The image depicts a room with a green floor and gray walls. The ceiling is blue, indicating a clear sky. No objects are visible in the room.''
    \item ``The image depicts a 3D environment with a gray wall, green grass floor, and a blue sky visible above. A blue box is located at 2.6 units and 0.1 degrees to the right, while a green ball is positioned at 3.5 units and 30.7 degrees to the right within the scene.''
    \item ``An empty room with a green grass floor and gray walls under a blue sky.''
    \item ``A blue box is visible at a distance of 4.1 units and an angle of 6.5 degrees to the left, while a green ball is located at a distance of 3.5 units and an angle of 29.9 degrees to the right. The scene is set on a green grass floor, with a grey wall in the background, under a blue sky.'
\end{itemize}

\section{RL Implementation Details}
\label{app:rl_details}

All Reinforcement Learning (RL) source policies—DQN, PPO, and SAC—were trained using the default hyperparameters and configurations provided by the Stable Baselines3 (SB3) library \cite{stable-baselines3}. No extensive hyperparameter tuning was performed for the source tasks, demonstrating the robustness of the source policies.

\section{Text-Conditioned VAE Architecture and Training}
\label{app:vae_details}

\subsection{Architecture}
Our Text-Conditioned Variational Autoencoder (VAE) is designed to separate visual content into a semantic, text-conditioned component and a spatially-structured latent variable $z$ that captures residual details (e.g., layout, pose, background). The framework consists of four main components: a visual encoder, a text encoder, a conditional decoder, and a set of discriminators.

\noindent \textbf{Visual Encoder.} The encoder $E_\phi$ is a ResNet-style Fully Convolutional Network (FCN). It processes an input image $x \in \mathbb{R}^{3 \times H \times W}$ through a series of downsampling blocks (comprising strided convolution, LayerNorm, and LeakyReLU) followed by three residual blocks. The encoder outputs a spatial feature map parameterizing a diagonal Gaussian distribution $q_\phi(z|x) = \mathcal{N}(\mu(x), \text{diag}(\sigma^2(x)))$, producing a spatial latent tensor $z \in \mathbb{R}^{C_z \times h \times w}$ (where $C_z=8$). Maintaining spatial dimensions allows the latent code to preserve localized visual structure.

\noindent \textbf{Text Encoder.} Textual descriptions are encoded using a frozen pre-trained LongCLIP model (\texttt{LongCLIP-GmP-ViT-L-14}) for unstructured captions. For structured captions, we utilized the standard CLIP model. We extract the last hidden state sequence $t \in \mathbb{R}^{L \times D_t}$ to capture fine-grained semantic details. A learnable MLP adapter projects these embeddings to the decoder's working dimension.

\noindent \textbf{Conditional Decoder.} The decoder $D_\theta(z, t)$ reconstructs the image $\hat{x}$ by progressively upsampling the latent $z$ while conditioning on text $t$. It is composed of hierarchical \textit{Cross-Attention FiLM Spatial Blocks}. In each block:
\begin{enumerate}
    \item \textbf{FiLM Modulation:} The spatial latent $z$ is upsampled to the features' resolution and mapped to affine parameters ($\gamma, \beta$) for Feature-wise Linear Modulation, allowing the latent structure to spatially modulate the features.
    \item \textbf{Cross-Attention:} A Multi-Head Cross-Attention layer allows the visual features to attend to the sequence of text embeddings $t$, injecting semantic information.
\end{enumerate}

\noindent \textbf{Discriminators.} The architecture includes two distinct adversarial modules used during training:
\begin{itemize}
    \item \textbf{Caption Discriminator ($D_\text{text}$):} To ensure the latent $z$ captures only visual information \textit{orthogonal} to the text, we employ an adversarial disentanglement module. This consists of a Multi-Layer Perceptron (MLP) that takes the flattened latent $z$ and attempts to predict the corresponding text embeddings $t$. A Gradient Reversal Layer (GRL) is placed before this discriminator, causing the encoder to learn representations that are invariant to the text (i.e., maximizing the discriminator's loss).
    \item \textbf{Image Discriminator ($D_\text{img}$):} A PatchGAN discriminator is used to enforce photorealism. It operates on local image patches to distinguish between real images $x$ and reconstructions $\hat{x}$, encouraging the decoder to generate high-frequency textures.
\end{itemize}

\subsection{Training Formulation}
The model is trained to minimize a composite objective function. The core foundation is the text-conditioned ELBO:
\begin{equation}
    \mathcal{L}_{VAE}(x, c_s) = \mathbb{E}_{q_{\phi}(z|x)}[\log p_{\theta}(x|z, e_{c_s})] - \beta \cdot D_{KL}(q_{\phi}(z|x)||p(z))
\end{equation}
where $p(z) \sim \mathcal{N}(0, I)$ and $\beta$ is the annealing term.

\noindent \textbf{Adversarial Disentanglement ($\mathcal{L}_{dis}$):} To ensure $z$ is orthogonal to the text, we use a contrastive adversarial loss. A discriminator $D_{\psi}$ predicts the text embedding $e_{c_s}$ from the latent $z$ (via GRL). It minimizes the InfoNCE loss:
\begin{equation}
    \mathcal{L}_{adv}(Z, E) = -\frac{1}{B} \sum_{i=1}^B \log \frac{\exp(\text{sim}(\hat{e}_{c_{s,i}}, e_{c_{s,i}}) / \tau)}{\sum_{j=1}^B \exp(\text{sim}(\hat{e}_{c_{s,i}}, e_{c_{s,j}}) / \tau)}
\end{equation}
where $\hat{e}_{c_{s,i}} = D_{\psi}(\text{GRL}(z_i))$ and $\tau$ is the temperature. The encoder maximizes this loss.

\noindent \textbf{Total Objective:} The full training objective combines these with perceptual and GAN losses for high-fidelity generation:
\begin{equation}
    \mathcal{L} = \mathcal{L}_\text{MSE} + \lambda_\text{LPIPS} \mathcal{L}_\text{LPIPS} + \lambda_\text{KL} \mathcal{L}_\text{KL} + \lambda_\text{dis} \mathcal{L}_{adv} + \lambda_\text{gan} \mathcal{L}_\text{gan}
\end{equation}

\noindent \textbf{Reconstruction ($\mathcal{L}_\text{MSE}$):} We calculate the pixel-wise Mean Squared Error (MSE) between the input $x$ and reconstruction $\hat{x}$ to ensure structural fidelity.

\noindent \textbf{Perceptual Loss ($\mathcal{L}_\text{LPIPS}$):} We optimize a VGG-based LPIPS perceptual loss to capture high-level semantic similarity and textural details that pixel-wise metrics may miss.

\noindent \textbf{KL Divergence ($\mathcal{L}_\text{KL}$):} The posterior is constrained towards a standard normal prior, $D_\text{KL}(q_\phi(z|x) || \mathcal{N}(0, I))$, using cyclical annealing to prevent posterior collapse.

\noindent \textbf{GAN Loss ($\mathcal{L}_\text{gan}$):} The Image Discriminator and Decoder are optimized via a Hinge loss adversarial objective to improve generation quality.

\subsection{Implementation Details}
We train the model using the AdamW optimizer with a OneCycleLR scheduler (max learning rate $2 \times 10^{-4}$, weight decay $0.01$). The implementation utilizes `LongCLIP-GmP-ViT-L-14` as the text backbone for natural scenes, while standard CLIP is used for structured environments. Hyperparameters are set to $\lambda_\text{KL}=1.5$, $\lambda_\text{dis}=15.0$, and $\lambda_\text{gan}=10.0$. Training is performed with a batch size of 32 for 500 epochs.

\section{Additional Qualitative Results}
\label{app:additional_results}
\subsection{Learning curves}
We present the detailed fine-tuning learning curves for the MiniWorld, MiniGrid and Fragile Object Manipulation environments in Figure \ref{fig:miniworld_curves} and \ref{fig:minigrid_curves} respectively. 

\begin{figure*}[h]
    \centering
    \begin{subfigure}[b]{0.32\textwidth}
        \centering
        \includegraphics[width=\textwidth]{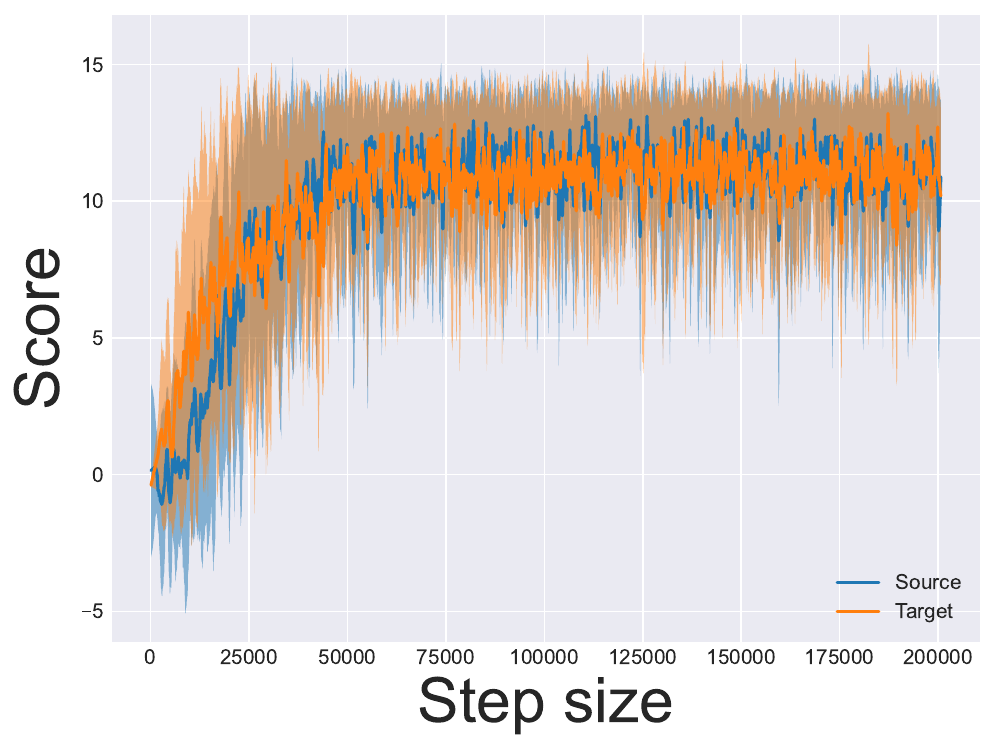}
        \caption{Case 1}
    \end{subfigure}
    \hfill
    \begin{subfigure}[b]{0.32\textwidth}
        \centering
        \includegraphics[width=\textwidth]{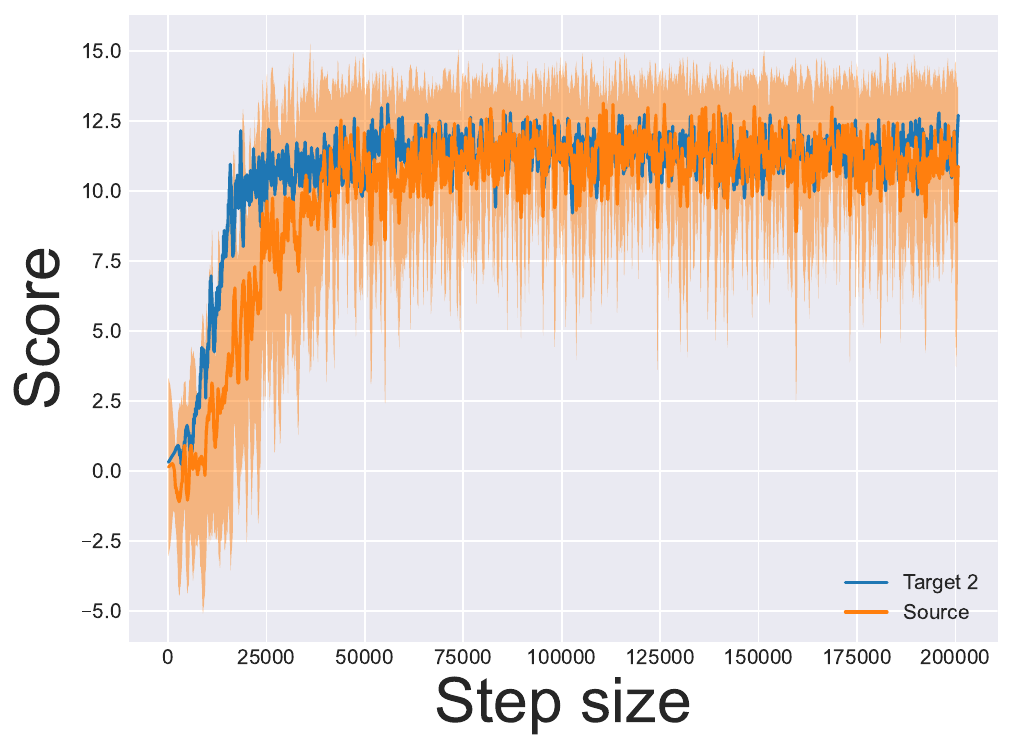}
        \caption{Case 2}
    \end{subfigure}
    \hfill
    \begin{subfigure}[b]{0.32\textwidth}
        \centering
        \includegraphics[width=\textwidth]{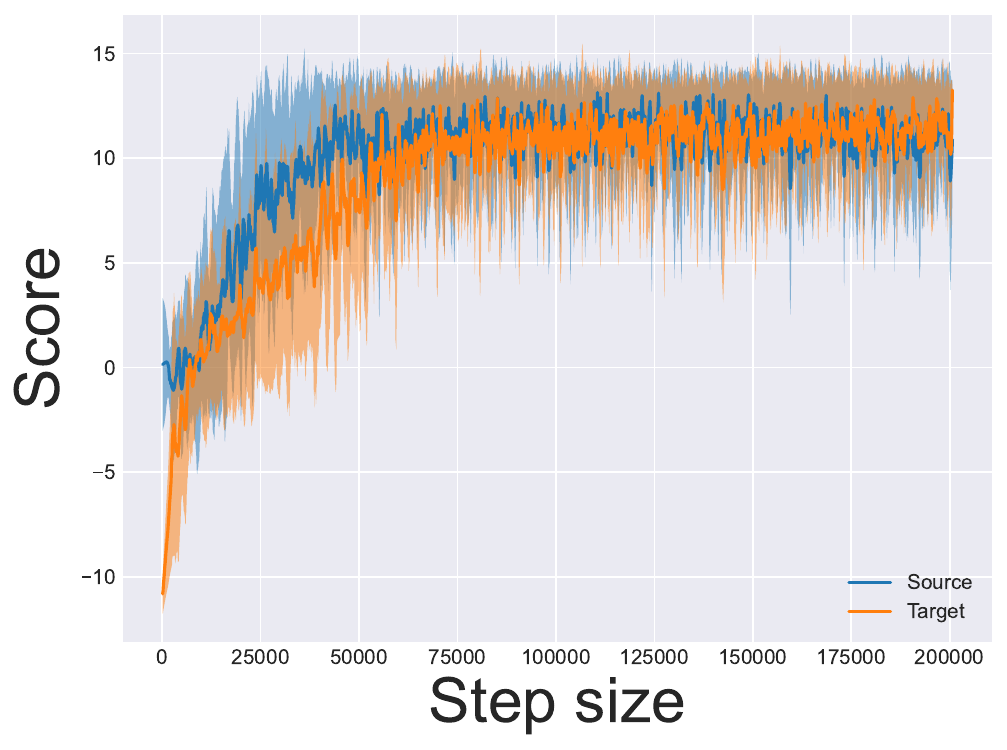}
        \caption{Case 3}
    \end{subfigure}
    \caption{PPO Fine-tuning in MiniWorld in each case.}
    \label{fig:miniworld_curves}
\end{figure*}

\begin{table}[h]
  \caption{Experimental results of ASPECT for MiniWorld environment using unstructured, VLM-generated captions. The method demonstrates robustness to noise and variability in the conditioning text.}
  \label{tab:miniworld_case_noise}
  \begin{center}
    \begin{small}
      \begin{sc}
        \resizebox{0.7\columnwidth}{!}{
        \begin{tabular}{lcc}
          \toprule
          Case & Rewarding Object Picked & Distractor Object Picked \\
          \midrule
          Case 1 & 8.60 $\pm$ 0.50 & 0.00 $\pm$ 0.00 \\
          Case 2 & 8.40 $\pm$ 0.40 & 0.00 $\pm$ 0.00 \\
          Case 3 & 8.60 $\pm$ 0.40 & 0.00 $\pm$ 0.00 \\
          \bottomrule
        \end{tabular}
        }
      \end{sc}
    \end{small}
  \end{center}
  \end{table} 
\label{app:additional_curves}

\begin{figure}[h]
    \centering
    \begin{subfigure}[b]{0.32\textwidth}
        \centering
        \includegraphics[width=\textwidth]{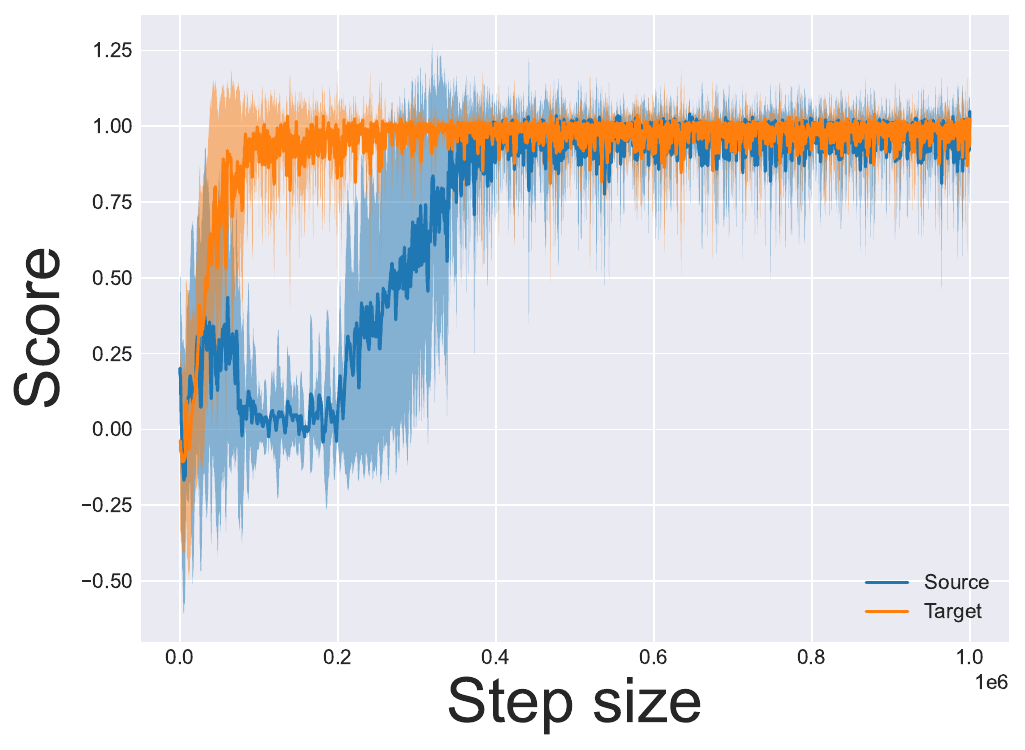}
        \caption{MiniGrid Case 1}
    \end{subfigure}
    \hfill
    \begin{subfigure}[b]{0.32\textwidth}
        \centering
        \includegraphics[width=\textwidth]{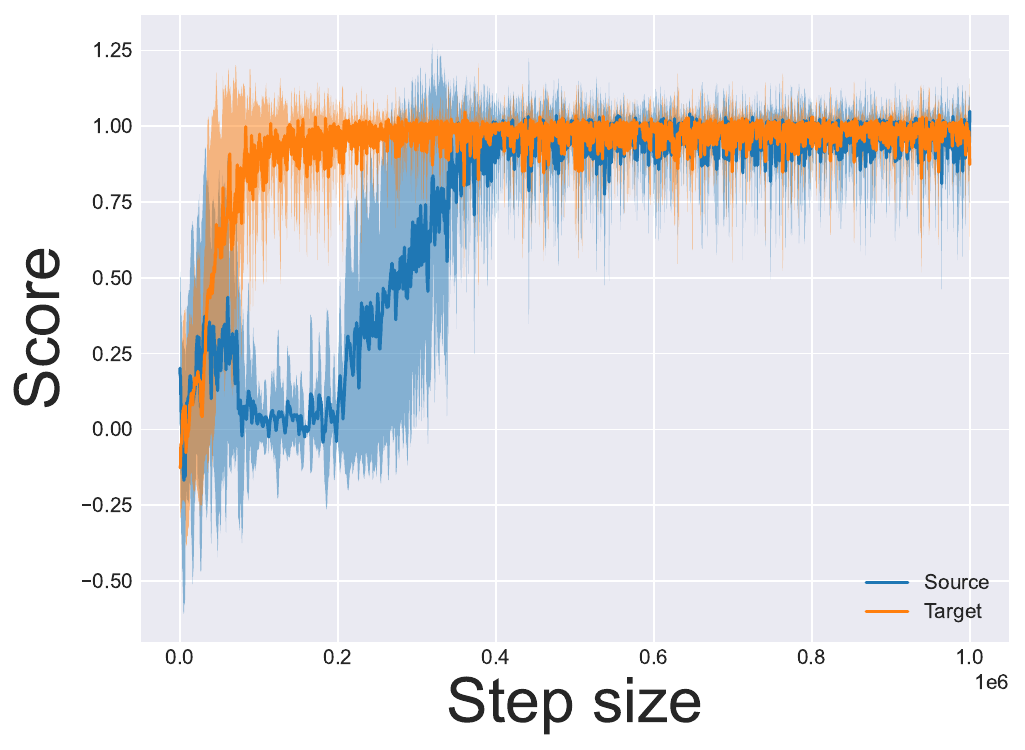}
        \caption{MiniGrid Case 2}
    \end{subfigure}
    \hfill
    \begin{subfigure}[b]{0.32\textwidth}
        \centering
        \includegraphics[width=\textwidth]{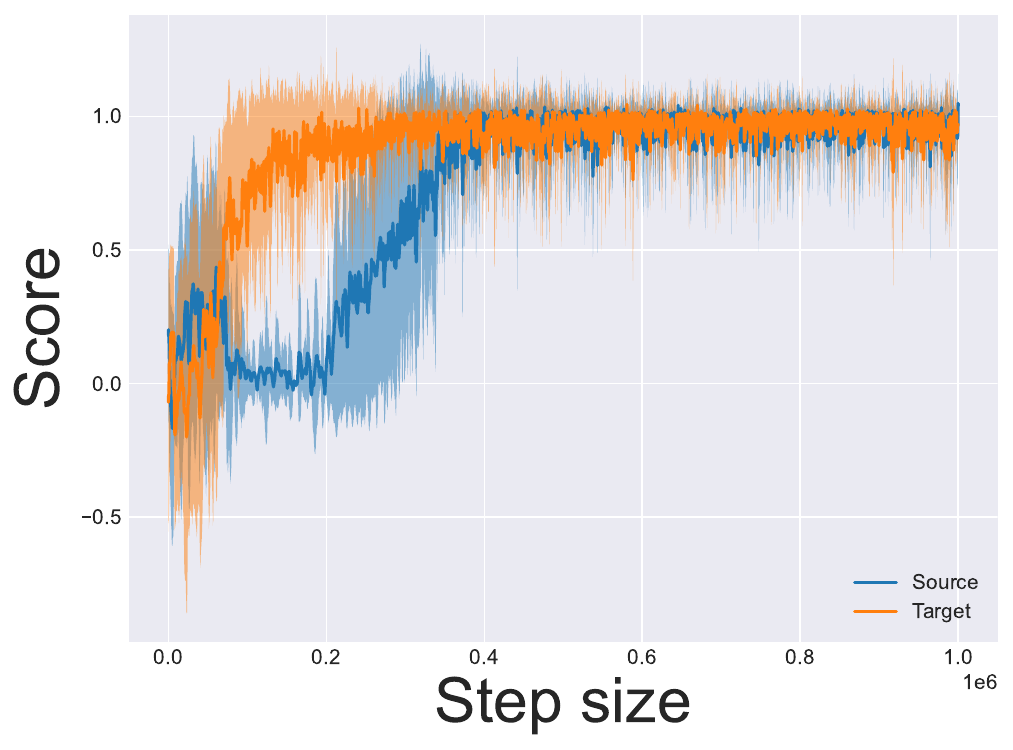}
        \caption{MiniGrid Case 3}
    \end{subfigure}
    \hfill
    \begin{subfigure}[b]{0.32\textwidth}
        \centering
        \includegraphics[width=\textwidth]{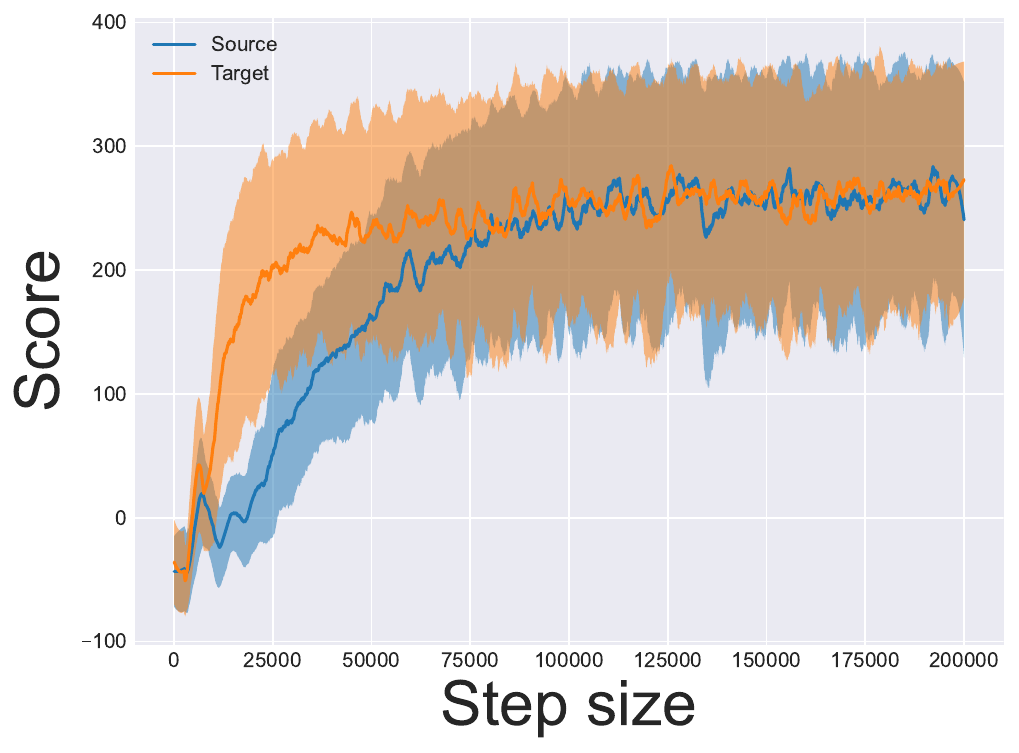}
        \caption{Manipulation}
        \label{fig:manipulation_curve}
    \end{subfigure}
    \caption{Fine-tuning curves for MiniGrid and Fragile Object Manipulation.}
    \label{fig:minigrid_curves}
\end{figure}
\label{app:qualitative}

The learning curve for SF in the MiniGrid environment is shown in Figure \ref{fig:sf_minigrid}. The SF agent is allowed to interact with different environments sequentially. Each dip in the learning curve indicates an environmental change.

\begin{figure}[h]
    \centering
    \includegraphics[width=0.5\textwidth]{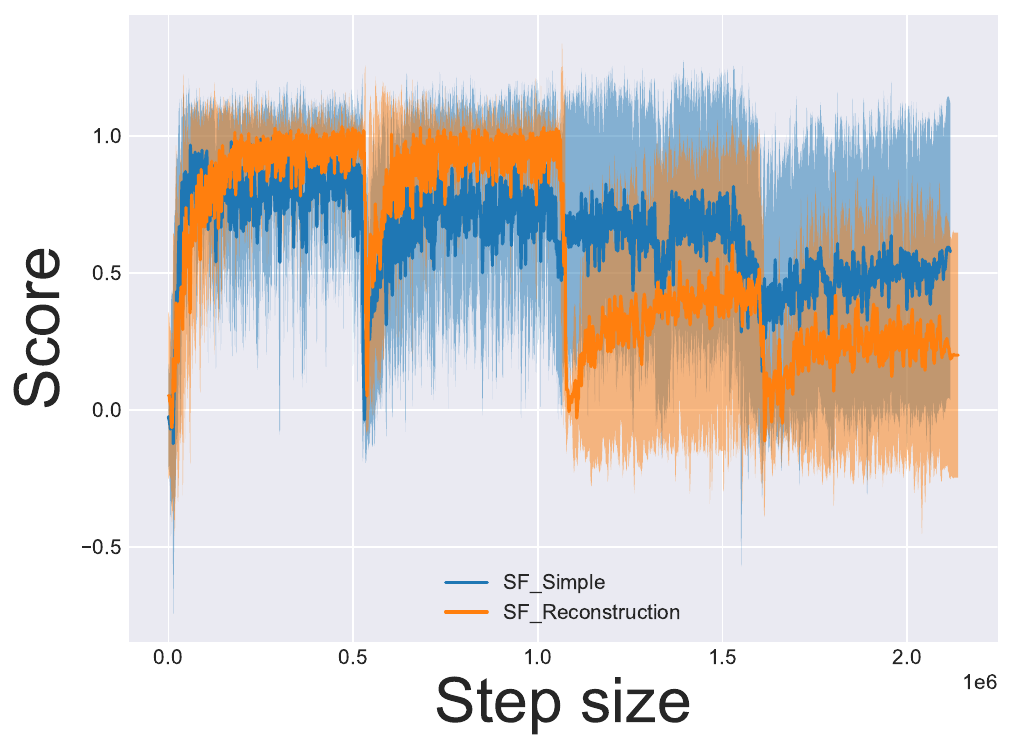}
    \caption{Learning curve for Successor Features (SF) in MiniGrid. The agent interacts with environments sequentially, with performance dips indicating task transfers.}
    \label{fig:sf_minigrid}
\end{figure}

The learning curve for SF in the MiniWorld environment is shown in Figure \ref{fig:sf_miniworld}. As indicated by the reward curve, the agent fails to learn any meaningful policy and the performance never improves across tasks.

\begin{figure}[h]
    \centering
    \includegraphics[width=0.5\textwidth]{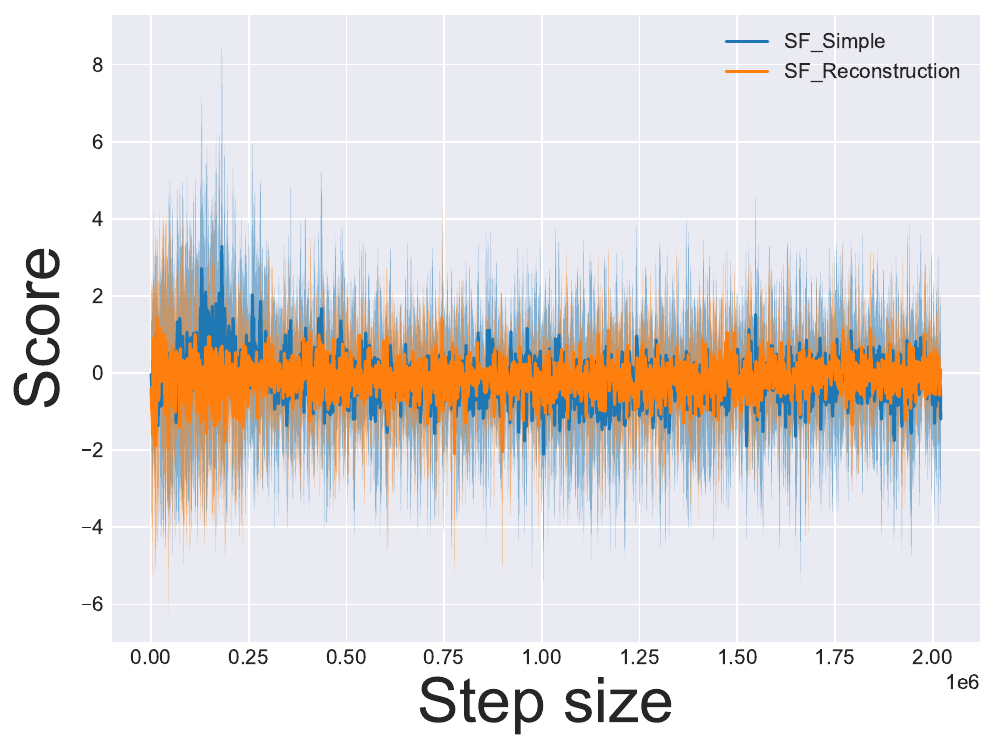}
    \caption{Learning curve for Successor Features (SF) in MiniWorld. The flat reward curve indicates a failure to learn or transfer to the target tasks.}
    \label{fig:sf_miniworld}
\end{figure}

\begin{figure}[h]
    \centering
    \begin{subfigure}[b]{0.23\textwidth}
        \centering
        \includegraphics[width=\textwidth]{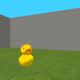}
        \caption{Original Image}
    \end{subfigure}
    \begin{subfigure}[b]{0.23\textwidth}
        \centering
        \includegraphics[width=\textwidth]{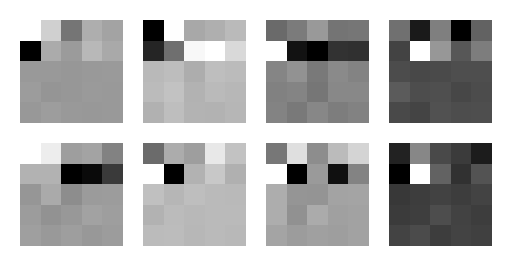}
        \caption{Original Latent}
    \end{subfigure}
    \begin{subfigure}[b]{0.23\textwidth}
        \centering
        \includegraphics[width=\textwidth]{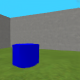}
        \caption{Remapped Image}
    \end{subfigure}
    \begin{subfigure}[b]{0.23\textwidth}
        \centering
        \includegraphics[width=\textwidth]{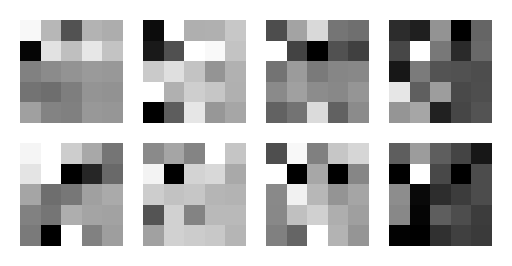}
        \caption{Remapped Latent}
    \end{subfigure}
    \caption{Visualization of spatial latent channels. The first two rows of the latent feature maps remain consistent between (b) and (d), capturing the preserved background structure, while other channels shift to reflect the semantic change from (a) to (c).}
    \label{fig:latent_viz}
\end{figure}

\begin{figure*}[h]
    \centering
    \includegraphics[width=\textwidth, height=0.42\textheight, keepaspectratio]{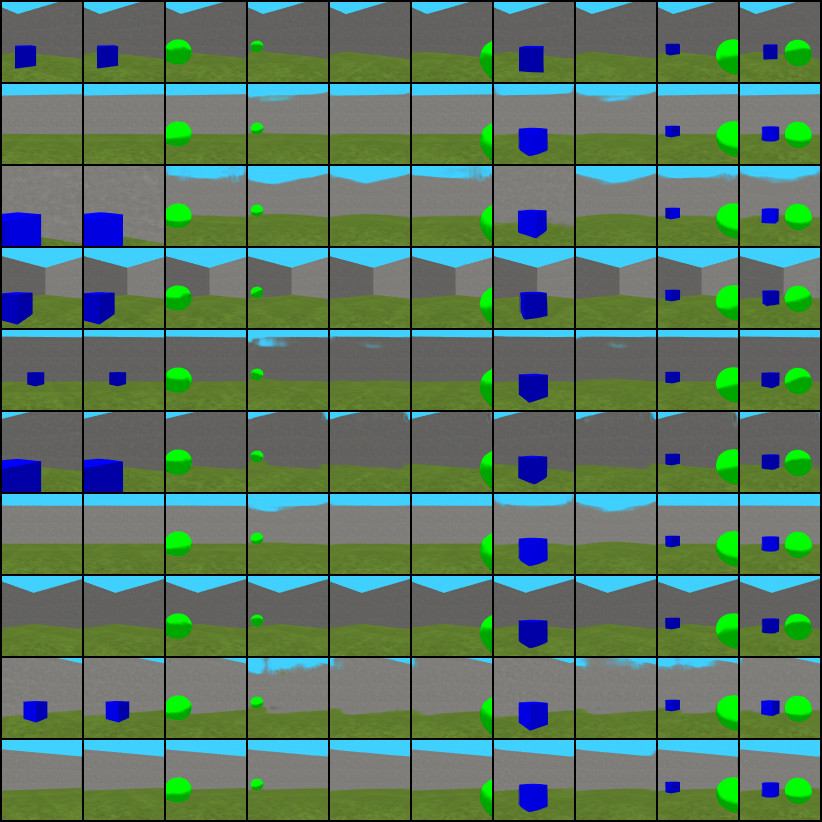}
    \includegraphics[width=\textwidth, height=0.42\textheight, keepaspectratio]{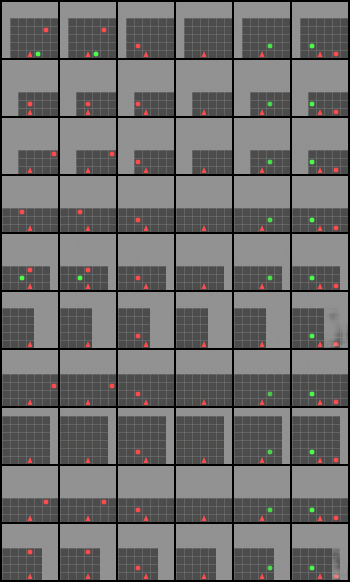}
    \caption{Text-guided image generation demonstrating latent disentanglement. Structural features captured in $z$ remain fixed (preserving the background layout) while the text modifies the semantic appearance. Left: MiniWorld, Right: MiniGrid.}
    \label{fig:remapping}
\end{figure*}

\begin{figure}[h]
    \centering
    \begin{subfigure}[b]{0.18\textwidth}
        \centering
        \textbf{Original}
    \end{subfigure}
    \hfill
    \begin{subfigure}[b]{0.18\textwidth}
        \centering
        \textbf{Imagined}
    \end{subfigure}
    \hfill
    \begin{subfigure}[b]{0.18\textwidth}
        \centering
        \textbf{Original}
    \end{subfigure}
    \hfill
    \begin{subfigure}[b]{0.18\textwidth}
        \centering
        \textbf{Imagined}
    \end{subfigure}
    
    \vspace{0.05cm}

    \begin{subfigure}[b]{0.18\textwidth}
        \centering
        \includegraphics[width=\textwidth]{Figures/frame1.png}
        \caption{Duckie in grass floor and concrete wall}
    \end{subfigure}
    \hfill
    \begin{subfigure}[b]{0.18\textwidth}
        \centering
        \includegraphics[width=\textwidth]{Figures/frame1_remapped.png}
        \caption{Blue box in source setting}
    \end{subfigure}
    \hfill
    \begin{subfigure}[b]{0.18\textwidth}
        \centering
        \includegraphics[width=\textwidth]{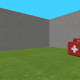}
        \caption{Medkit in grass floor and concrete wall}
    \end{subfigure}
    \hfill
    \begin{subfigure}[b]{0.18\textwidth}
        \centering
        \includegraphics[width=\textwidth]{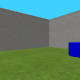}
        \caption{Blue box in source setting}
    \end{subfigure}
    
    \vspace{0.1cm}

    \begin{subfigure}[b]{0.18\textwidth}
        \centering
        \includegraphics[width=\textwidth]{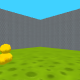}
        \caption{Duckie in slime floor and metal grill wall}
    \end{subfigure}
    \hfill
    \begin{subfigure}[b]{0.18\textwidth}
        \centering
        \includegraphics[width=\textwidth]{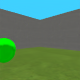}
        \caption{Green ball in source setting}
    \end{subfigure}
    \hfill
    \begin{subfigure}[b]{0.18\textwidth}
        \centering
        \includegraphics[width=\textwidth]{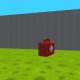}
        \caption{Medkit in slime floor and metal grill wall}
    \end{subfigure}
    \hfill
    \begin{subfigure}[b]{0.18\textwidth}
        \centering
        \includegraphics[width=\textwidth]{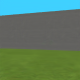}
        \caption{Green ball in source setting}
    \end{subfigure}
    
    \vspace{0.1cm}

    \begin{subfigure}[b]{0.18\textwidth}
        \centering
        \includegraphics[width=\textwidth]{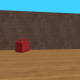}
        \caption{MedKit in wodden floor and brick wall}
    \end{subfigure}
    \hfill
    \begin{subfigure}[b]{0.18\textwidth}
        \centering
        \includegraphics[width=\textwidth]{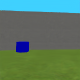}
        \caption{Blue box in source setting}
    \end{subfigure}
    \hfill
    \begin{subfigure}[b]{0.18\textwidth}
        \centering
        \includegraphics[width=\textwidth]{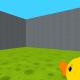}
        \caption{Duckie in slime floor and metal grill wall}
    \end{subfigure}
    \hfill
    \begin{subfigure}[b]{0.18\textwidth}
        \centering
        \includegraphics[width=\textwidth]{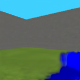}
        \caption{Blue box in source setting}
    \end{subfigure}
    
    \vspace{0.1cm}
    
    \begin{subfigure}[b]{0.18\textwidth}
        \centering
        \includegraphics[width=\textwidth]{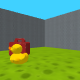}
        \caption{Duckie and MedKit in slime floor and metal grill wall}
    \end{subfigure}
    \hfill
    \begin{subfigure}[b]{0.18\textwidth}
        \centering
        \includegraphics[width=\textwidth]{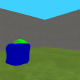}
        \caption{Blue box and Green ball in source setting}
    \end{subfigure}
    \hfill
    \begin{subfigure}[b]{0.18\textwidth}
        \centering
        \includegraphics[width=\textwidth]{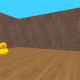}
        \caption{Duckie in wooden floor and brick wall}
    \end{subfigure}
    \hfill
    \begin{subfigure}[b]{0.18\textwidth}
        \centering
        \includegraphics[width=\textwidth]{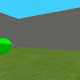}
        \caption{Green ball in source setting}
    \end{subfigure}

    \vspace{0.1cm}
    
    \begin{subfigure}[b]{0.18\textwidth}
        \centering
        \includegraphics[width=\textwidth]{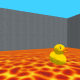}
        \caption{Duckie in lava floor and metal grill wall}
    \end{subfigure}
    \hfill
    \begin{subfigure}[b]{0.18\textwidth}
        \centering
        \includegraphics[width=\textwidth]{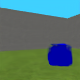}
        \caption{Blue box in source setting}
    \end{subfigure}
    \hfill
    \begin{subfigure}[b]{0.18\textwidth}
        \centering
        \includegraphics[width=\textwidth]{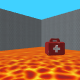}
        \caption{MedKit in lava floor and metal grill wall}
    \end{subfigure}
    \hfill
    \begin{subfigure}[b]{0.18\textwidth}
        \centering
        \includegraphics[width=\textwidth]{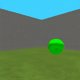}
        \caption{Green ball in source setting}
    \end{subfigure}
    
    \caption{Comparison of real target observations and the corresponding source-aligned imagined states generated by ASPECT. We display 10 examples demonstrating the robust visual remapping capability.}
    \label{fig:imagination_viz}
\end{figure}

\subsection{Disentanglement of Layout and Semantics}
\label{app:disentanglement}
Figure \ref{fig:remapping} demonstrates the text-guided image generation capabilities of our VAE, highlighting the disentanglement between structural and semantic features. The first column displays the original observation, and the second column shows its reconstruction. Subsequent columns show generations conditioned on different text descriptions while keeping the spatial latent $z$ fixed. As observed, the structural layout (e.g., background geometry, walls, floor) remains consistent across all generations, captured by the unchanged latent $z$. Meanwhile, the semantic content (textures, colors, object identities) adapts to the varying text prompts. This confirms that our architecture effectively disentangles the spatial layout (encoded in $z$) from the semantic attributes (controlled by the text embedding).

To further validate this disentanglement, we visualized the spatial latent features. Figure \ref{fig:latent_viz} displays the original and remapped images alongside their corresponding 8-channel spatial latents. Notably, the first two rows of the latent visualization (representing specific feature channels) remain identical between the original and remapped states. These channels correspond to the structural background features, which are preserved by the architecture, while the text-conditioned channels adapt to the new semantic prompts. This visual evidence reinforces that our model successfully isolates background structure from semantic object identity.

\subsection{Imagination in different unseen settings}
\label{sec:unseen_setting}
We provide visual examples of the imagination process in Figure \ref{fig:imagination_viz}. These qualitative results demonstrate how the agent hallucinates the target observation (containing unseen rooms and objects) to match the source task, enabling zero-shot transfer.

\subsection{Failure Cases}
\label{app:failure_cases}

While ASPECT demonstrates robust zero-shot generalization, we identify specific failure modes in the imagination process. Figure \ref{fig:failure_cases} illustrates these cases. Common issues include artefacts in the generated background or object when the target object is positioned too close to the camera, and occasional instances where the remapped object fails to materialize in the imagined scene.

\begin{figure}[h]
    \centering
    \begin{subfigure}[b]{0.35\textwidth}
        \centering
        \textbf{Original}
    \end{subfigure}
    \hspace{0.5cm}
    \begin{subfigure}[b]{0.35\textwidth}
        \centering
        \textbf{Imagined}
    \end{subfigure}
    
    \vspace{0.1cm}

    \begin{subfigure}[b]{0.35\textwidth}
        \centering
        \includegraphics[width=\textwidth]{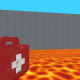}
        \caption{Object too close}
    \end{subfigure}
    \hspace{0.5cm}
    \begin{subfigure}[b]{0.35\textwidth}
        \centering
        \includegraphics[width=\textwidth]{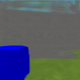}
        \caption{Background artefacts}
    \end{subfigure}
    
    \vspace{0.2cm}

    \begin{subfigure}[b]{0.35\textwidth}
        \centering
        \includegraphics[width=\textwidth]{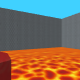}
        \caption{Partial object close up}
    \end{subfigure}
    \hspace{0.5cm}
    \begin{subfigure}[b]{0.35\textwidth}
        \centering
        \includegraphics[width=\textwidth]{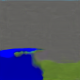}
        \caption{Artefacts in object/background}
    \end{subfigure}
    
    \vspace{0.2cm}

    \begin{subfigure}[b]{0.35\textwidth}
        \centering
        \includegraphics[width=\textwidth]{Figures/frame4.png}
        \caption{Target visible}
    \end{subfigure}
    \hspace{0.5cm}
    \begin{subfigure}[b]{0.35\textwidth}
        \centering
        \includegraphics[width=\textwidth]{Figures/frame4_remapped.png}
        \caption{Remapped object missing}
    \end{subfigure}
    
    \caption{Visualisation of failure cases. The model struggles with extreme close-ups, leading to generation artefacts, and occasionally fails to generate the source object even when the target is visible.}
    \label{fig:failure_cases}
\end{figure}

\end{document}